\documentclass{article}

\usepackage[preprint]{neurips_2025}

\usepackage[utf8]{inputenc} 
\usepackage[T1]{fontenc}    
\usepackage{hyperref}       
\usepackage{url}            
\usepackage{booktabs}       
\usepackage{amsfonts}       
\usepackage{nicefrac}       
\usepackage{microtype}      
\usepackage{xcolor}         

\usepackage{amsmath,amsthm,amstext,amsfonts,amssymb,mathrsfs}
\usepackage{dashrule}
\usepackage{mathtools}

\newcommand{\dasheddline}{\noindent\hdashrule[0.5ex]{\linewidth}{0.5pt}{3pt 2pt}}

\newtheorem{lemma}{Lemma}

\newtheorem{assumption}{Assumption}

\newtheorem{theorem}{Theorem}
\newtheorem{definition}{Definition}

\usepackage{textcase}
\usepackage{graphicx,subcaption}

\newcommand{\E}{\mathbb{E}}

\newcommand{\R}{\mathbb{R}}

\newcommand{\cA}{\mathcal{A}}

\newcommand{\cD}{\mathcal{D}}

\newcommand{\cP}{\mathcal{P}}

\newcommand{\cX}{\mathcal{X}}
\newcommand{\cY}{\mathcal{Y}}

\renewcommand{\epsilon}{\varepsilon}
\renewcommand{\phi}{\varphi}

\definecolor{powderpink}{RGB}{245,210,211} 
\definecolor{pastelgreen}{RGB}{189,208,196}
\definecolor{pastelblue}{RGB}{154,183,211}
\definecolor{lilac}{RGB}{190, 173, 201}
\definecolor{babypink}{rgb}{0.96, 0.76, 0.76}
\definecolor{bazaar}{rgb}{0.6, 0.47, 0.48}
\definecolor{cambridgeblue}{rgb}{0.64, 0.76, 0.68}
\definecolor{glaucous}{rgb}{0.38, 0.51, 0.71}
\definecolor{fuzzywuzzy}{rgb}{0.8, 0.4, 0.4}
\usepackage[T1]{fontenc}
\usepackage[scaled=0.88]{beramono}    
\usepackage{listings}
\definecolor{cadmiumgreen}{rgb}{0.0, 0.42, 0.24}
\renewcommand{\arraystretch}{1.15} 

\usepackage{array}
\newcolumntype{Y}{>{\scriptsize}p{1.7cm}}
\usepackage{enumitem}
\usepackage{amsmath}
\usepackage{algorithm}
\usepackage{algpseudocode}
\usepackage{wrapfig}
\usepackage[most]{tcolorbox}

\title{PITA: Preference-Guided Inference-Time Alignment for LLM Post-Training}

\author{%
  Sarat Chandra Bobbili$^{1}$ \quad
  Ujwal Dinesha$^{1}$ \quad
  Dheeraj Narasimha$^{2}$ \quad
  Srinivas Shakkottai$^{1}$ \\
  $^{1}$Texas A\&M University, $^{2}$ Inria.\\
}

\begin{document}

\maketitle

\begin{abstract}
Inference-time alignment enables large language models (LLMs) to generate outputs aligned with end-user preferences without further training. Recent post-training methods achieve this by using small guidance models to modify token generation during inference. These methods typically optimize a reward function KL-regularized by the original LLM taken as the reference policy.   A critical limitation, however, is their dependence on a pre-trained reward model, which requires fitting to human preference feedback, a potentially unstable process. In contrast, we introduce PITA, a novel framework that integrates preference feedback directly into the LLM's token generation, eliminating the need for a reward model.  PITA learns a small preference-based guidance policy to modify token probabilities at inference time without LLM fine-tuning, bypassing the pre-trained reward model dependency.   The problem is framed as identifying an underlying preference distribution, solved through iterative refinement of the  preference-based guidance model. We evaluate PITA across diverse tasks, including mathematical reasoning, TL;DR summarization, and sentiment classification, demonstrating its effectiveness in aligning LLM outputs with user preferences.

\end{abstract}

\section{Introduction}

Large language model (LLM) alignment traditionally relies on additional training, such as fine-tuning with human feedback, to ensure outputs conform to human desires. However, fine-tuning LLMs is computationally expensive and complex. This has motivated the study of inference-time alignment, where the pre-trained LLM’s weights remain frozen, and alignment is achieved by steering the model’s decoding process to meet end-user needs. Inference-time alignment aims to align generated outputs with user criteria (e.g., helpfulness, correctness, sentiment) without updating the LLM's parameters. By intervening only during generation, this approach avoids retraining overhead and enables the alignment of opaque models whose internal weights are inaccessible.

Recent post-training alignment methods employ small guidance models or value functions to influence token generation during decoding. These approaches frame decoding as an optimization problem, where the model maximizes a user-provided reward while maintaining proximity to its original style. They optimize a KL-regularized objective, balancing the base LLM’s log-likelihood with the reward from the guidance model. The reward signal originates from a pre-trained model that scores candidate tokens based on their anticipated impact on the alignment of subsequent tokens with the reward criteria. These methods integrate the guidance model into the decoding loop, adjusting the next-token distribution at each step. Although computationally intensive, these techniques enhance alignment without modifying the LLM’s weights.

\begin{figure}[!ht]
  \centering
  \includegraphics[width=\textwidth]{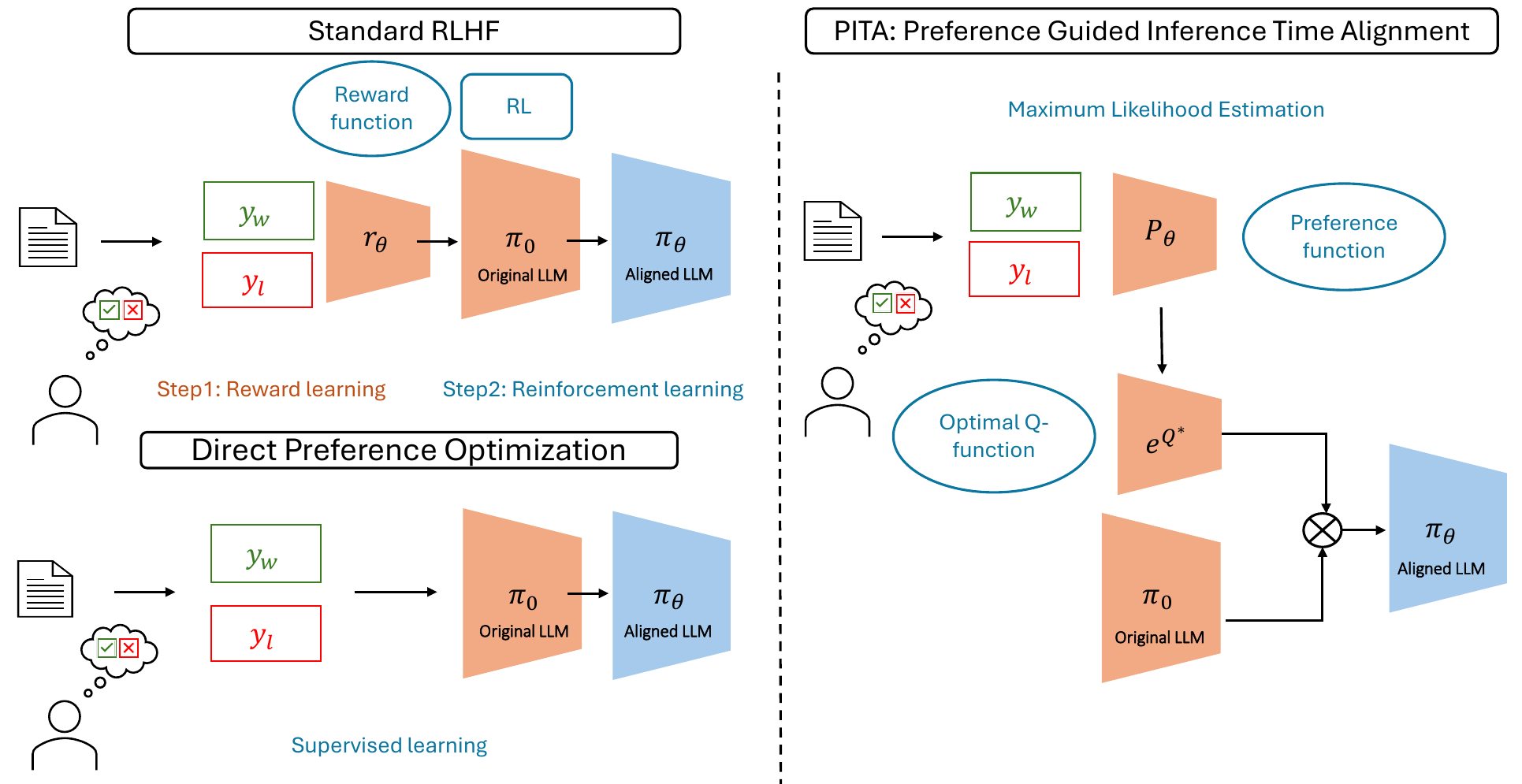} 
  \caption{Overview of PITA training and inference time implementation. During training, our algorithm learns a Action-Value function $Q^*$ directly from the preference information. At Inference, the output distribution of the original LLM, $\pi_{\mathrm{ref}}$, is modified with exponentially weighted $Q^*$-values}
  \label{fig:pita-vs-rlhf}
  \vspace{-0.2in}
\end{figure}

Reward-model-guided alignment methods face several challenges. First, they assume the availability of a pre-trained reward model, which is often not available. Reward models are typically learned from human-labeled preference data. However, preference data might be noisy or inconsistent, hindering the accurate learning of a reward model that reflects true user intent. Moreover, training a reward model can be unstable and prone to brittleness, where minor inaccuracies in the model can cause the LLM to generate misaligned outputs. These issues highlight the difficulty of learning a reliable reward model from preference data, motivating the need to eliminate an explicit reward predictor.

In this work, we introduce \textbf{PITA: Preference-Guided Inference-Time Alignment}, a novel framework that eliminates the need for a reward model in LLM post-training. PITA directly uses preference feedback during inference to guide token generation. The key innovation is the use of a small, preference-based guidance policy that operates alongside a frozen LLM and adjusts token probabilities in real-time, without modifying the LLM’s weights. This guidance policy is trained using user preference signals, significantly reducing the requirement for extensive human-labeled data and avoiding the complexity of reward model training.


PITA frames the alignment task as learning an underlying preference distribution over the model's outputs. It uses stochastic search and iterative refinement to identify tokens that align better with user preferences. During each decoding round, the guidance policy explores multiple continuations and updates based on user feedback, progressively refining the model’s predictions. This process guides the model to converge towards outputs that better match user preferences, while keeping the base LLM unchanged. Figure~\ref{fig:pita-vs-rlhf} illustrates the differences between PITA and standard post-training.

We first provide a performance analysis of PITA, using the notion of a linear function approximator for the preference model. As the learned preference model aligns with the true one, PITA's guidance model aligns with the true reward-to-go, resulting in sub-linear regret. We then evaluate PITA on diverse tasks, including mathematical reasoning and sentiment alignment datasets. An example of the  outputs generated by the reference LLM vs. PITA for a math reasoning problem from the GSM8K dataset is shown in Figure~\ref{fig:math}.  Here, PITA aligns the model’s steps with user-preferred solutions without fine-tuning, improving accuracy and consistency.    In sentiment tasks, PITA ensures outputs meet user-defined sentiment or safety criteria, such as avoiding toxic language.

\begin{figure}[ht]
\centering
\begin{minipage}{\textwidth}
\begin{tcolorbox}[colback=white, colframe=black]

\textbf{Problem from GSM8K:} Christina is planning a birthday party and needs .75 gift bags per invited guest, because 1/4 of attendees don’t show up. She invited 16 friends. Gift bags are \$2 each. How much will she spend?

\dasheddline

\textbf{Incorrect Solution from $\pi^{\text{ref}}$:} \textcolor{red!80!black}{Let’s break this down step by step:
1. Christina invited 16 friends. 2. Since 1/4 of attendees don’t show up, she’ll only need to give gift bags to 3/4 of the invited guests. To find this amount, multiply 16 by 3/4: 16 × (3/4) = 12. 3. Since she needs 0.75 gift bags per invited guest, she’ll need 0.75 × 12 = 9 gift bags. 4. Each gift bag costs \$2, so the total cost will be 9 × \$2 = \$18.  So, Christina will spend $\boxed{\$18}$ on gift bags.}\\ 
\dasheddline

\textbf{Correct Solution from PITA:} \textcolor{blue!30!green}{Since Christina invited 16 friends, and 1/4 of attendees don’t show up, the number of friends who will actually attend the party is 16 - (1/4)*16 = 16 - 4 = 12.
Since Christina needs 0.75 gift bags per invited guest, and she invited 16 friends, she needs 0.75*16 = 12 gift bags.
The cost of each gift bag is \$2, so the total cost of the gift bags is 12*2 = \$24. Final Answer: The final answer is \boxed{24}. I hope it is correct.
}





\end{tcolorbox}
\end{minipage}
\caption{Comparison of Incorrect Solution from the Reference model and the Correct Solution from PITA to a math reasoning problem.}
\label{fig:math}
\end{figure}

We compare PITA against two reward-based alignment schemes: Q\# ~\citep{zhou2025q} uses a known reward, and Q\#-HF where we learn reward from preference data via maximum likelihood estimation, followed by alignment using Q\#. PITA performs comparably to Q\#, while outperforming the learned-reward-based alignment method (as well as the base LLM), demonstrating its superior efficiency and effectiveness in utilizing preference data. PITA achieves this level of performance without requiring a pre-trained reward model, highlighting its advantage in preference-guided inference-time alignment.
\section{Related Work}

Preference learning has a long history in RL and bandit literature~\citep{akrour2012april,auer2014algorithmic,sadigh2017active}. Alongside the empirical success of training LLMs to generate outputs aligned with human preferences referred to as Reinforcement Learning with Human Feedback (RLHF), the theoretical aspects of preference learning have been extensively studied~\citep{chambers2021recovering,shah2016estimation,heckel2018approximate}. Reward-based methods have been shown to be unstable, and relative preferences are generally easier to obtain than absolute reward scores. Direct preference optimization (DPO) has emerged as a promising approach, addressing the inconsistencies of reward-based methods in reflecting true human preferences~\citep{knox2022models}. Several works characterize the alignment problem as a general preference optimization objective, which is the focus of our work. For instance, ~\citet{azar2024general} define a general RLHF objective, with RLHF and DPO as specific examples. In Contrastive Preference Learning (CPL)~\citep{hejna2023contrastive}, a novel approach models human preferences based on regret rather than reward, enabling scalable, off-policy learning without requiring reward functions or reinforcement learning. Additionally, ~\citet{ye2024online} proposes a general RLHF framework without assuming a reward function or Bradley-Terry preferences, using a reverse-KL minimax game between LLMs, and introduces sample-efficient algorithms for both offline and online learning with empirical validation. An extended version of the related work is provided in Appendix~\ref{app:related-work-extd}.
\section{Preliminaries}\label{sec:preliminaries}
We begin by introducing the finite-horizon Markov Decision Process (MDP) framework that formalizes the generation process of language modeling.

A finite-horizon MDP is formulated as $\mathcal{M} = (\mathcal{P}, \mathcal{S}, \mathcal{A}, T, \rho)$, where $\mathcal{S}$ is a finite set of states, $\mathcal{A}$ is a set of actions, $\mathcal{P}: \mathcal{S} \times \mathcal{A} \rightarrow \Delta(\mathcal{S})$ is the transition kernel, $T \in \mathbb{N}$ is the maximal length of the episode, and $\rho$ denotes the starting state distribution. A policy in MDPs, denoted by $\pi : \mathcal{S} \rightarrow \Delta(\mathcal{A})$, maps each state to a distribution over actions. The dynamics of the MDP can be described as follows: Initially, the starting state $s_0$ is sampled from the initial distribution $\rho$. At each step $t$, the policy selects an action $a_t$ based on the current state $s_t$. The process transitions to the next state $s_{t+1}$, which is sampled from the distribution $\mathcal{P}(\cdot \mid s_t, a_t)$. This process continues until a termination condition is met within the horizon length $T$.

\textbf{LLMs as Deterministic MDPs:} In text generation applications using an LLM, $\pi$ is the LLM's policy to choose the next token (action) given the current sequence of tokens (state). The transition kernel is deterministic; each new state is obtained by concatenating the previous tokens with the current token chosen by the policy. Each sentence generated by the policy is termed a trajectory $\tau = (s_0,a_0,s_1,a_1,\cdots)$, where $s_0 := x$ is the initial instruction sampled from the starting state distribution $\rho$, and $a_t \sim \pi(\cdot |s_t)$. Therefore, at each step $t$, the state $s_t = (x | a_{1:t-1})$ consists of the instruction $x$ and tokens generated up to $t-1$. The generation process ends when the LLM policy samples a special end-of-sequence token or the trajectory reaches a predefined length limit, characterized by $T$. The score of a trajectory generated by the LLM is defined using a reward function $r(s,a):\mathcal{S} \times \mathcal{A} \rightarrow \mathbb{R}$. The reward function is typically sparse with $r(s_t,a_t) = 0$ for all $t \in [T-2]$, quantifying the desirability of the final output of the LLM for the given instruction $x$. An LLM policy can be learned using $T$-horizon reinforcement learning by maximizing the expected reward for a given policy defined as $J(\pi) = \mathbb{E}_{\tau \sim \pi} [R(\tau)] = \mathbb{E}_{x \sim \rho}[V^{\pi}(x)]$, where $V^{\pi}(s) = \mathbb{E}_{\pi} \left[ \sum_{t=0}^{T-1} r(s_t,a_t) | s_0 = s \right]$ is the value of a state $s$ under policy $\pi$. Similarly, the state-action value function and the advantage function for a state-action pair $(s,a)$ are given by $Q^{\pi}(s,a) = \mathbb{E}_{\pi} \left[ \sum_{t=0}^{T-1} r(s_t,a_t) | s_0 = s, a_0 = a \right]$, and $A^{\pi}(s,a) = Q^{\pi}(s,a) - V^{\pi}(s)$ respectively.

\textbf{RLHF/DPO:} RLHF and DPO are a few commonly used approaches to align LLMs with human preference data. The RLHF pipeline consists of three steps: a supervised fine-tuning phase where a pre-trained LLM is fine-tuned with supervised learning for downstream task(s) to obtain a model $\pi_{ref}$, a reward modeling phase to estimate a reward function for the LLMs generations, and a fine-tuning phase to find the optimal policy which maximizes the expected reward while not deviating too far from the reference model $\pi_{\mathrm{ref}}$.

\textbf{Reward Modeling:} The reference model $\pi_{\mathrm{ref}}$ is prompted with an instruction $x$ to obtain an answer pair $(y_1,y_2)$. The completions are presented to human annotators who express their preference for one answer over the other, denoted as $y_w \succ y_l | x$ where $y_w$ and $y_l$ denote the preferred
answer. The preferences are assumed to be generated according to the Bradley-Terry model \citep{bradley1952rank}, given by
\begin{equation}
    \mathcal{P}^*(y_w \succ y_l) = \sigma(r^*(x,y_w) - r^*(x,y_l)).
\end{equation}
where $r^*(\cdot,\cdot)$ is a latent reward model unknown to the LLM. Assuming access to an offline dataset $\mathcal{D} = \{(x^{(i)},y_w^{(i)},y_l^{(i)})\}_{i=1}^n$ of preference information generated from $\mathcal{P}^*$, the reward model is parameterized as $r_{\theta}(x,y)$ and estimated via maximum likelihood. The estimation problem is formalized as a binary classification problem with the negative log-likelihood loss given by
\begin{equation}
    \mathcal{L}(r_{\theta},\mathcal{D}) = - \mathbb{E}_{(x,y_w,y_l) \sim \mathcal{D}} \left[ \log{\sigma(r_{\theta}(x,y_w) - r_{\theta}(x,y_l))}\right]
\end{equation}

\textbf{RL Fine-Tuning:} The reference model $\pi_{\mathrm{ref}}$ is fine-tuned with the learned reward function $r_{\theta}$, which is posed as the following KL-regularized constraint optimization problem:
\begin{equation}\label{eqn:kl-obj}
    \max_{\pi} \mathbb{E}_{x \sim \mathcal{D},\, y \sim \pi(y \mid x)} \left[ r_\theta(x, y) \right] - \eta \mathrm{KL} \left[ \pi(y \mid x) \,\|\, \pi_{\mathrm{ref}}(y \mid x) \right]
\end{equation}
where $\eta$ is a parameter that controls the deviation of the learned $\pi$ from the reference policy $\pi_{\mathrm{ref}}$, and the KL divergence is defined as $\mathrm{KL}(p || q) = \mathbb{E}_{z \sim p} \left[ \log{p(z)/q(z)}\right]$.

\textbf{DPO:} In \citet{rafailov2023direct}, the RL fine-tuning step is solved as a supervised learning problem to directly learn the optimal policy using preference information. The key idea is that the optimal policy, as obtained by solving the KL-constrained reward maximization objective, is a function of the implicit reward model and has the following form:
\begin{equation}
    \pi^*_r(y \mid x) = \frac{1}{Z(x)} \pi_{\mathrm{ref}}(y \mid x) \exp\left( \frac{1}{\eta} r^*(x, y) \right)
\end{equation}
where $Z(x) = \sum_y \pi_{\mathrm{ref}}(y|x) \exp(\frac{1}{\eta}r^*(x,y))$ is the partition function. Although estimation of the partition function is difficult, DPO circumvents this challenge by expressing the BT preference model as a function of the optimal policy $\pi^*$ given by
\begin{equation}
    \mathcal{P}^*(y_1 \succ y_2 \mid x) = \frac{1}{1 + \exp\left( \eta \log \frac{\pi^*(y_2 \mid x)}{\pi_{\mathrm{ref}}(y_2 \mid x)} - \eta \log \frac{\pi^*(y_1 \mid x)}{\pi_{\mathrm{ref}}(y_1 \mid x)} \right)}
\end{equation}
Similar to the reward modeling approach, under the above reformulation of the preference model, the policy objective is defined as:
\begin{equation}
\mathcal{L}_{\text{DPO}}(\pi; \pi_{\text{ref}}) = - \mathbb{E}_{(x, y_w, y_l) \sim \mathcal{D}} \left[
\log \sigma \left(
\eta \log \frac{\pi(y_w \mid x)}{\pi_{\text{ref}}(y_w \mid x)} -
\eta \log \frac{\pi(y_l \mid x)}{\pi_{\text{ref}}(y_l \mid x)}
\right)
\right].
\end{equation}

\section{PITA: Preference-Guided Inference-Time Alignment} \label{sec:PITA}

In this section, we present our algorithm for test-time scaling of LLMs using a preference model. The key idea is to leverage the structure of the optimal policy under $\pi^*$ as a function of the state-action value function $Q^*$ in deterministic MDPs \citep{zhou2025q}. This framework applies specifically to LLM fine-tuning. To facilitate understanding, we present the relevant results.

For a given parameter $\eta > 0$, the soft-value function $V^{\pi,\eta}$ of a policy $\pi$ is defined as:
\begin{equation}\label{eqn:klreg-soft-obj}
    V^{\pi,\eta}(s) = \mathbb{E}_{\pi} \left[ \sum_{t=0}^{T-1} r(s_t, a_t) - \eta \, \mathrm{KL}(\pi(s_t) \parallel \pi_{\mathrm{ref}}(s_t)) \Big| s_0 = s \right]
\end{equation}
It can be shown that the KL-regularized RL objective is solvable using the soft Bellman equations. In particular, the optimal policy is given by \cite{rafailov2024r}:
\begin{align}\label{eqns:soft-bellman}
V_{T}^{\star, \eta}(s) &= 0, \nonumber\\
Q_t^{\star, \eta}(s, a) &= r(s, a) + \mathbb{E}_{s' \sim P(\cdot \mid s,a)}\left[ V_{t+1}^{\star, \eta}(s') \right], \nonumber\\
\pi^{\star, \eta}(a_t \mid s_t) &\propto \pi_{\text{ref}}(a_t \mid s_t) \exp\left( \eta^{-1} Q_t^{\star, \eta}(s_t, a_t) \right), \\
V_t^{\star, \eta}(s) &= \eta \ln \mathbb{E}_{a \sim \pi_{\text{ref}}(\cdot \mid s)} \left[ \exp\left( \eta^{-1} Q_t^{\star, \eta}(s, a) \right) \right] \nonumber.
\end{align}

Here, $Q^{\star, \eta}$ represents the soft-state-action value function under the optimal policy $\pi^*$. We specialize our discussion to the deterministic transition structure of the LLM generative process. In the sparse reward setting, $V^{\pi,\eta}, Q^{\pi,\eta}$ are functions of the conditional reward distributions of the generations under $\pi_{\mathrm{ref}}$ \citep{zhou2025q}. Specifically:
\begin{align}
V_{t}^{\star, \eta}(s) &= \eta \ln \mathbb{E}_{\pi_{\text{ref}}} \left[ \exp \left( \eta^{-1}  r(s_{T-1}, a_{T-1}) \right) \middle| s_t = s \right], \\
Q_{t}^{\star, \eta}(s,a) &= \eta \ln \mathbb{E}_{\pi_{\text{ref}}} \left[ \exp \left( \eta^{-1}  r(s_{T-1}, a_{T-1}) \right) \middle| s_t = s, a_t = a \right] \label{eqn:optimal-soft-q}.
\end{align}
For the sake of completeness, we present the proof of the above results in Appendix~\ref{app:theory}. The Q\# approach utilizes a policy of the form given in (\ref{eqns:soft-bellman}) with $Q_{t}^{\star, \eta}(s,a)$ as given in (\ref{eqn:optimal-soft-q}).

\textbf{Preference-based Q Value function:} Unlike Q\#, the goal of PITA is to directly learn an optimal LLM policy using preference data.

We first derive the soft-state-action value function and the value function in terms of the preference distribution. For a given context $x$, let $y^{\mathrm{ref}}_x$ be a unique \textit{reference} completion starting from state $x$. The preference distribution of any completion $y_x$ given $x$, sampled according to a policy $\pi$ over $y^{\mathrm{ref}}_x$, is given by $\mathcal{P}^*(y_x \succ y^{\mathrm{ref}}_x)$.
\begin{theorem}\label{thm:pref-Q}
    Let $\Psi(x) \triangleq \log{\frac{x}{(1-x)}}$. Then, assuming the BT preference model, we have
    \begin{align}
        V_{t}^{\star, \eta}(s) &= r(y^{\mathrm{ref}}_s) + \eta \ln \mathbb{E}_{y_s \sim \pi_{\text{ref}}} \left[ \exp \left( \eta^{-1}  \Psi( \mathcal{P}^*(y_s \succ y^{\mathrm{ref}}_s)) \right) \middle| s_t = s \right], \nonumber \\
Q_{t}^{\star, \eta}(s,a) &= r(y^{\mathrm{ref}}_s) + \eta \ln \mathbb{E}_{y_{s'} \sim \pi_{\text{ref}}} \left[ \exp \left( \eta^{-1} \Psi( \mathcal{P}^*(y_{s'} \succ y^{\mathrm{ref}}_s)) \right) \middle| s_t = s, a_t = a \right],
    \end{align}
where $s' = (s | a)$ is the state obtained by concatenating $a$ to state $s$.
\end{theorem}

The proof of Theorem~\ref{thm:pref-Q} is provided in Appendix~\ref{app:theory}. The above theorem shows that the expressions $Q_{t}^{\star, \eta}, V_{t}^{\star, \eta}$ are functionals of the preference distribution of completions according to the reference policy $\pi_{\mathrm{ref}}$ with an offset term. While the reward function $r^*$ is unknown, the optimal policy is a function of the advantage function, making it sufficient to evaluate the preference distribution. The reference completion for any intermediate state $s$ must be unique. In the rest of the discussion, we choose $y^{\mathrm{ref}}_s$ to be the completion obtained via greedy decoding using policy $\pi_{\mathrm{ref}}$.

\textbf{Preference distribution over greedy decoding:} The preference probability $\mathcal{P}^*(y_x \succ y^{\mathrm{ref}}_x)$ represents the expected number of `wins' of a completion over greedy decoding for a given context. Since the reference greedy completion $y_x^{\mathrm{ref}}$ is unique, using the reward-equivalence property of the Bradley-Terry model, we derive the following lemma.

\begin{lemma}\label{lem:reward-equivalence}
     Let $r^*$ be the implicit reward function. Then, for any given context $x$ and policy $\pi$, there exists a reward function $\tilde{r}$ such that the preference probability of generations $y_x \sim \pi$ over the greedy completion is given by $\sigma(\tilde{r}(y))$.
\end{lemma}

The proof of Lemma~\ref{lem:reward-equivalence} is given in Appendix~\ref{app:theory}. The above reward equivalence is crucial for developing our algorithm. It states that the preference probability is solely a function of the completion and context. We can parameterize $\mathcal{P}^{\theta}(y_x)$ using a neural network, which estimates the expected wins $\mathcal{P}^*(y_x \succ y^{\mathrm{ref}}_x)$ over greedy decoding. From the optimal policy derived from the soft-Bellman equations, we have:
\begin{equation}
\pi^{\star, \eta}(a \mid s) \propto \pi_{\text{ref}}(a \mid s)  \mathbb{E}_{y_{s'} \sim \pi_{\text{ref}}} \left[ \exp \left( \eta^{-1} \Psi( \mathcal{P}^*(y_{s'} \succ y^{\mathrm{ref}}_s)) \right) \middle| s_t = s, a_t = a \right]
\end{equation}
This naturally motivates us to consider parametrized policies of the form $\pi^{\theta}$, such as:
\begin{equation}
\pi^{\theta, \eta}(a \mid s) \propto \pi_{\text{ref}}(a \mid s)  \mathbb{E}_{y_{s'} \sim \pi_{\text{ref}}} \left[ \exp \left( \eta^{-1} \Psi( \mathcal{P}^{\theta}(y_{s'} \succ y^{\mathrm{ref}}_s)) \right) \middle| s_t = s, a_t = a \right]
\end{equation}
where $s' = (s | a)$ is the state obtained by concatenating $s$ with action $a$.

\subsection{Algorithm}
In this section, we present PITA, a preference-guided inference-time alignment algorithm for solving the KL-regularized RLHF objective. PITA is an iterative algorithm where progressively better policies are learned using a parametrized preference function. We describe the PITA algorithm in the context of LLMs, and in Section~\ref{sec:pita-convergence}, we provide convergence guarantees under linear reward approximation.

Note that the optimal policy is obtained by guiding generations from the reference policy $\pi_{\mathrm{ref}}$ using the regularized soft-action value function $Q^*$. However, evaluating $Q^*$ is challenging unless the distribution over preference functions for each generation under $\pi_{\mathrm{ref}}$ is known \textit{a priori}. Assuming a parametric family, one can apply distributional techniques, such as maximum likelihood estimation, for parameter inference. To address this, we adopt certainty equivalence as an approximation in our practical implementation.

Let $\hat{\cP}^{\theta}: \cX \rightarrow [0,1]$ be the preference function that maps a given state $x := (s,a) \in \cX$ to the expected preference under $\pi_{\mathrm{ref}}$, i.e. $\E_{y_x \sim \pi_{\text{ref}}}[\cP^*(y_x \succ y^{\mathrm{ref}}_s)]$. The natural loss function is the binary cross-entropy loss (BCE):
\begin{equation}
    \mathcal{L}(o_y,\hat{\cP}) = - o_y \log{\hat{\cP}} - (1 - o_y) \log{( 1 - \hat{\cP})}
\end{equation}
where $o_y = \mathbf{1}_{y \succ y^{\mathrm{ref}}}$ is the indicator function of the completion $y$ preferred over the reference completion $y^{\mathrm{ref}}$.

The algorithm iteratively updates the estimator parameters $\theta$ using new preference information collected by the induced policy $\pi_k \triangleq \pi^{\theta_k,\eta}$. Specifically, the data collection strategy follows a rollout step using $\pi^{k}$ for $t$ steps until a context $s_t$ is sampled. Then, a completion is randomly sampled starting from this context along with greedy decoding using $\pi_{\text{ref}}$. The preference observed between the two completions is a sample from $\cP^*$. These preference samples are added to the dataset, and parameters are updated via gradient descent. The procedure is repeated until convergence. The pseudocode for the algorithm is detailed in Algorithm~\ref{alg:PITA}.

\begin{algorithm}[H]
\caption{Preference-Guided Inference-Time Alignment (PITA)}
\label{alg:PITA}
\begin{algorithmic}[1]
\State \textbf{Input:} Reference policy $\pi^{\text{ref}}$. Initialize $\theta_0$ and dataset $\mathcal{D} = \emptyset$.
\For{$k = 1, 2, \dots$ until convergence}
    \State Let $\pi^{k-1} \leftarrow \pi^{{\theta_{k-1}},\eta}$ be the policy induced by $\hat{\cP}^{\theta_{k-1}}$.
    \For{$i = 1, 2, \dots, N$}
        \State Sample a rollout step uniformly: $t \sim \mathrm{Unif}([T])$.
        \State Roll in with $\pi^{k-1}$ for $t$ steps to obtain the context $s_t$.
        \State Sample the reference response $y^{\mathrm{ref}}_{s_t}$ starting from context $s_t$ using greedy decoding.
        \State Sample $M$ responses $y_{s_t}^1, \dots, y_{s_t}^M$ with $\pi^{\text{ref}}$ starting from $s_t$.
        \State Add $(s_t, y_{s_t}^i, y^{\mathrm{ref}}_{s_t}, \mathbf{1}_{y_i \succ y^{\mathrm{ref}}_{s_t}})$ to $\mathcal{D}$ for all $i \in [M]$.
    \EndFor
    \State Update $\theta^k$ using Maximum Likelihood Estimation on the aggregated data:
    \[
    \theta^{k} \leftarrow \arg\min_{\theta} \mathbb{E}_{(x,y,y') \sim \mathcal{D}} \left[ \mathcal{L}(\mathbf{1}_{y \succ y'}, \hat{\cP}^{\theta}) \right]
    \]
\EndFor
\State \textbf{Output:} Final $\theta^k$.
\end{algorithmic}
\end{algorithm}

We note that our algorithm belongs to the class of value-based algorithms for post-training LLMs \citep{mudgal2023controlled,han2024value,zhou2025q}. In methods like CD \citep{mudgal2023controlled} and VAS \citep{han2024value}, the learned Q-function is a non-regularized version of the optimal $Q^*$. In contrast, Q\# \citep{zhou2025q} aims to learn the optimal state-action value function using distributional reinforcement learning \citep{bellemare2023distributional}.  Our iterative algorithm shares design principles with Q\#, but it learns the optimal action-value function directly from preference data, eliminating the need for a pre-trained reward model.
\section{Theoretical Results on PITA Performance}\label{sec:pita-convergence}

In this section, we present theoretical performance guarantees for PITA.  We focus on the case of a linear reward class.  All proofs appear in Appendix~\ref{app:convergence-proofs}.
\begin{assumption}\label{assum:theta_bound}
The reward lies in the family of linear functions $r_{\theta}(y) = \theta^\top \Phi(y)$ for some known $\Phi$ with $\sup \|\Phi(\cdot)\|_2 \leq L$. Let $\theta^\star$ be the true parameter. To ensure the identifiability of $\theta^\star$, we assume $\theta^\star \in \Theta_B$, where
\[
\Theta_B = \left\{ \theta \in \mathbb{R}^d \mid \langle \mathbf{1}, \theta \rangle = 0, \|\theta\|_2 \leq B \right\}.
\]
\end{assumption}

\subsection{Preference Modeling in PITA}
We denote by $\Phi(x)$ the $d$-dimensional embedding of a state $x$, where $\Phi : \cX \rightarrow \R^d$. The embedding function is assumed to be known to the learner. In the context of LLMs, $\Phi$ is obtained from the hidden representation of the second-to-last layer of the model. The set of all completions starting from a given state $x$ is denoted by $\cY^{\pi}_x = \left\{y: y \sim \pi(\cdot | x)\right\}$. Thus, the preference distribution $\cP^*(y_x)$ for a completion $y_x \in \cY^{\pi}_x$ is given by
\begin{equation}
    \cP^*(y_x) = \cP^*(y_x \succ y_x^{\pi}) = \sigma(\langle \theta^*, \Phi(y_x) \rangle - \langle \theta^*, \Phi(y_x^{\pi}) \rangle),
\end{equation}
where $y_x^{\pi}$ is the greedy completion under $\pi$. Since the greedy completion is unique for a fixed policy $\pi$ and starting context $x$, we can rewrite the parametric class of preference distributions for PITA as:
\begin{equation}
    \cP^{\theta}(y_x) = \sigma(\langle \theta, \Phi(y_x) - \Phi(y_x^{\pi_{\mathrm{ref}}}) \rangle) 
\end{equation}
Since $y_x^{\pi_{\mathrm{ref}}}$ is fixed, we can, without loss of generality, renormalize $\Phi(y_x^{\pi_{\mathrm{ref}}})$ to 0, as all preference data for a given $x$ is measured relative to $y_x^{\pi_{\mathrm{ref}}}$. Thus, we have:
\begin{equation}
    \cP^{\theta}(y_x) = \sigma(\langle \theta, \Phi(y_x) \rangle).
\end{equation}
Here, $\theta \in \Theta_B$, and $y_x \in \cY^{\pi_{\mathrm{ref}}}_x$. In the following, we omit the context $x$ from our notation, as the results apply on a per-context basis.

\textbf{Maximum Likelihood Estimation (MLE).}  
In the above preference modeling, a natural way to compute an estimator $\theta_k$ given completion pairs $\{(x_i, y_i, y^{\mathrm{ref}}_i)\}_{i=1}^n$ and their preference feedback values $\{o_i\}_{i=1}^n$ is via maximum likelihood estimation (MLE). MLE minimizes the negative log-likelihood function, defined as:
\begin{align}
\hat{\theta} &\in \arg \min_{\theta \in \Theta_B} \mathcal{L}_{\mathcal{D}}(\theta), \\
\mathcal{L}_{\mathcal{D}}(\theta)
&= - \sum_{i=1}^n o_i \log{\sigma(\langle \theta, \Phi(y_i) \rangle)} + 
(1 - o_i) \cdot \log{1 - \sigma(\langle \theta, \Phi(y_i) \rangle)}.
\end{align}

Following results from dueling bandits and reinforcement learning \citep{faury2020improved, shah2016estimation, pacchiano2021dueling}, we have the following result on convergence of the MLE. 
\begin{lemma}\label{eqn:mle}
Under Assumption~\ref{assum:theta_bound}, for any $\lambda > 0$, with probability at least $1 - \delta$, we have
\begin{equation}
\left\|\hat{\theta} - \theta^*\right\|_{\Sigma_{\mathcal{D}} + \lambda I} \leq C \cdot \sqrt{\frac{d + \log(1/\delta)}{\gamma^2 n} + \lambda B^2} .
\end{equation}
where $\Sigma_{\mathcal{D}} = \frac{1}{n} \sum_{i=1}^n \Phi(y_i) \Phi(y_i)^\top$, 
$\gamma = \frac{1}{2 + \exp(-LB) + \exp(LB)}.$
\end{lemma}

\subsection{Regret Bound}
We now state our main PAC result. We make the following assumption on the boundedness of the embedding function.
\begin{assumption}\label{assum:emb-norm-bound}
Consider the space of all policies of the form 
$\Pi_{\theta} \triangleq \{ \pi_{\theta}: \pi_{\theta} \propto \pi_{\mathrm{ref}} \exp{(\eta^{-1} \Phi)} \}$.
We assume $\left\|(\Sigma_{\mathcal{D}} + \lambda I)^{-1/2} \mathbb{E}_{y \sim \pi} [\Phi(y)] \right\|_2$ is bounded for all $\pi \in \Pi_{\theta}.$ 
\end{assumption}

Then we have the following regret bound on the performance of PITA.
\begin{theorem}\label{thm:pita-regret-bound}
For any $\lambda > 0$, with probability at least $1 - \delta$,
\begin{equation}
\sum_{k=1}^K (V^{\star,\eta} - V^{\theta_k,\eta}) \leq C \sum_{k=1}^K  \cdot \left( \sqrt{\frac{d + \log(1/\delta)}{\gamma^2 n_k} + \lambda B^2}  \cdot \sup \left\|(\Sigma_{\mathcal{D}} + \lambda I)^{-1/2} \mathbb{E}_{y \sim \pi} [\Phi(y)] \right\|_2 \right),
\end{equation}
where $n_k$ is the number of preference samples collected up to iteration $k$, and $\pi \in \Pi_{\theta}$.
\end{theorem}
Note that Theorem~\ref{thm:pita-regret-bound} demonstrates sub-linear regret with respect to the number of samples used for maximum likelihood estimation. In~\cite{zhu2023principled}, the authors show the sub-optimality of policies estimated using a pessimistic MLE algorithm under the same linear reward model as ours. While standard concentration bounds for MLE estimators in linear reward settings are well-established, our convergence analysis follows a substantially different approach. The key insight is that the optimal soft-action value function $Q^*$ can be expressed as a simple functional of preference data generated by a fixed reference policy. This transformation makes the analysis significantly more straightforward.

\section{Experimental Evaluation}\label{sec:experiments}

We empirically evaluate the proposed \textbf{PITA} algorithm for mathematical reasoning tasks, sentiment generation, and compare it against a strong reward-model baseline.  Our study focuses on three key questions: \textbf{(Q1)} Does PITA improve task accuracy (\textit{pass@1} and \textit{maj@8}) over a baseline using post-training Q\# with a learned preference reward model? \textbf{(Q2)} Does the KL divergence to the reference policy remain controlled? \textbf{(Q3)} How well does the learned reward model discriminate between correct and incorrect generations?


In this section, we provide an overview of the experimental results, while additional details on datasets, training, and evaluation criteria are provided in Appendix~\ref{app:experiments}. All experiments were run on machines equipped with a single NVIDIA A100 80~GB GPU. Our code is available at \url{https://github.com/SaratBobbili/pita}.

\subsection{Experimental Setup}\label{subsec:setup}

\textbf{Models and Datasets.} 
We perform PITA training using the following models: Meta-Llama-3-8B-Instruct, Meta-Llama-3.2-1B-Instruct, TinyLlama-1.1B-Chat-v1.0, and GPT-2 small and medium models.  
We evaluate our algorithm on three diverse tasks with varying difficulty levels: arithmetic reasoning (GSM8K~\cite{cobbe2021training}), star-graph reasoning~\cite{bachmann2024pitfalls}, and sentiment generation~\cite{chaudhary2025riskaversefinetuninglargelanguage}.  
We use the Llama 3~\cite{grattafiori2024llama} family of models to train our preference function, as they are highly competitive in arithmetic reasoning tasks~\cite{zhou2025q}. The GPT-2 small and medium models are used for training on the star-graph reasoning task, following the framework in~\cite{bachmann2024pitfalls}. For sentiment generation, we use the TinyLlama 1.1B model, following~\cite{chaudhary2025riskaversefinetuninglargelanguage}. 

\textbf{Metrics.} We report \emph{pass@1} (exact-match accuracy of the highest-ranking sample), \emph{maj@8} (majority-vote accuracy over 8 samples), and the forward KL divergence $\text{KL}(\pi\,\Vert\,\pi_{\text{ref}})$ computed on the \textsc{GSM8K} test set. Lower KL values indicate closer adherence to the reference policy.

\subsection{Algorithms}

We compare the performance of the following algorithms against the base reference model $\pi_{\mathrm{ref}}$:
\begin{itemize}[leftmargin=1.4em]
    \item \textbf{Q\#}: The algorithm described in~\cite{zhou2025q}, trained directly using access to true binary reward labels.  
    \item \textbf{Q\#-HF}: A {\em two-stage} procedure in which a reward model is first trained from preference pairs (so-called ``human feedback''), and this learned reward is then used in the Q\# algorithm to obtain the final policy.  
    \item \textbf{PITA (ours)}: The algorithm described in Section~4, which learns a soft \(Q\)-function directly from preference data and induces the final policy without requiring a separate reward model.
\end{itemize}

\subsection{Performance on GSM8K}\label{subsec:main-results}

We present our results for the performance of our algorithm on the arithmetic reasoning dataset GSM8K~\cite{cobbe2021training}. GSM8K (Grade School Math 8K) is a dataset containing 8.5K high-quality, linguistically diverse math word problems aimed at the grade-school level. These problems, designed for question-answering tasks, require 2 to 8 steps and involve basic arithmetic operations. For test set evaluation, we use the entire GSM8K test set to evaluate the performance of the trained PITA model~\cite{lightman2023let,wang2023math}.

\textbf{Reward-Model Calibration in Q\#-HF:} We first train the reward model needed by Q\#-HF  as an intermediate step before using the Q\# algorithm to obtain a policy.  We utilize the full preference data set in order to train this reward model.    Figure~\ref{fig:reward_hist_correct} contrasts reward scores assigned by the BT model to correct versus incorrect generations produced by the baseline. A larger separation indicates stronger alignment between the reward and task success. The performance of Q\#-HF is highly sensitive to the reward training process, as we see in an ablation study of reward training from preference data presented in Appendix~\ref{app:additonal-experiments}.

\begin{figure}[H]
    \centering
    \includegraphics[width=\linewidth]{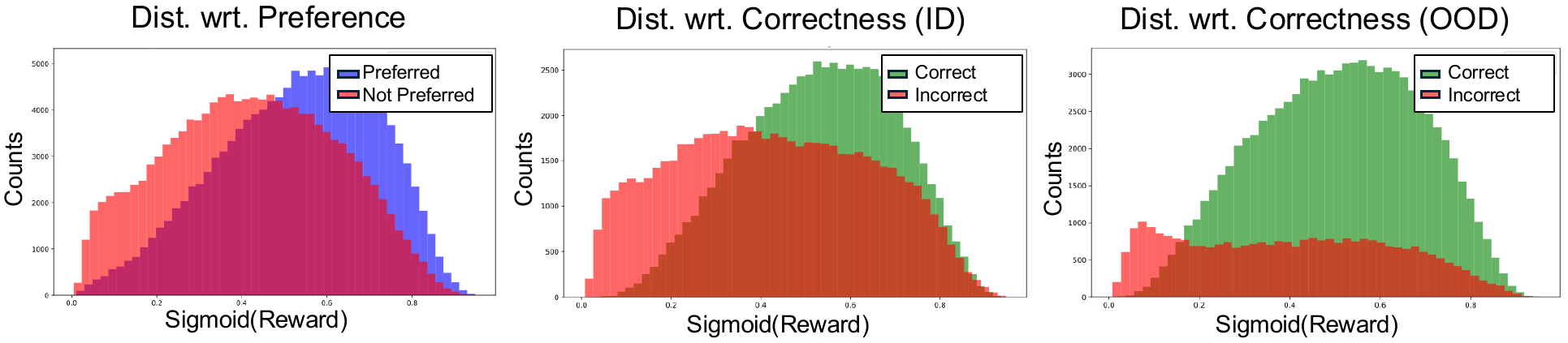}
    \caption{Distribution of learned reward scores w.r.t. preferred/correct (green) and not preferred/incorrect (red) for ID (in-dist.)/OOD (out-of-dist.) data. The reward function training is brittle and sensitive to the amount of data used.}
    \label{fig:reward_hist_correct}
\end{figure}

\textbf{Main Results:} In Table~\ref{tab:gsm8k_main}, we compare the performance of PITA and other value-guided algorithms to the reference model $\pi_{\mathrm{ref}}$, which is chosen to be the Llama 3 8B model. We observe that PITA significantly outperforms Llama 3 on both \textbf{pass@1} accuracy and \textbf{maj@8} evaluation metrics. Additionally, the performance gap between PITA and Q\# appears to be small. Training details for PITA are provided in Appendix~\ref{app:experiments}.  Thus, Table~\ref{tab:gsm8k_main} shows that PITA attains performance on par with Q\# even without access to true reward labels.

 Despite using a learned reward model with full availability of preference data, Q\#-HF fails to perform as well as PITA, highlighting the brittleness of reward estimation in this setting.

 Additionally, the KL divergence of PITA is higher than that of Q\# and Q\#-HF. This is because the language model used to provide preference feedback in PITA is OpenMathInstruct-2~\cite{toshniwal2024openmathinstruct}, which is distinct from the rule-based reward function score used in Q\#. This highlights the tradeoff between alignment with the reference policy and the flexibility of preference-based feedback.

\begin{table}[H]
  \centering
  \caption{GSM8K accuracy (\%) and KL divergence.}
  \label{tab:gsm8k_main}
  \begingroup
  \setlength{\tabcolsep}{6pt}
  \renewcommand{\arraystretch}{1.2}
  \begin{tabular}{lccc}
    \toprule
    \textbf{Algorithm} & \textbf{pass@1} & \textbf{maj1@8} & \textbf{KL}↓ \\
    \midrule
    $\pi_{\text{ref}}$  & 69.1 & 85.8 & - \\
    Q\#                 & 78.4 & 88.1 & 2.65 \\
    Q\#-HF              & 67.23 & 78.09 & 2.48 \\
    \midrule
    \textbf{PITA (ours)}& 77.11 & 86.20 & 6.15 \\
    \bottomrule
  \end{tabular}
  \endgroup
\end{table}

\textbf{Discussion:}   We observe that PITA significantly outperforms the Q\#-HF baseline on both accuracy metrics. The poor performance of Q\#-HF can be attributed to inconsistencies in the reward estimation process, which failed to provide stable guidance during training. This demonstrates that iterative policy optimization with sparse rewards, as used in PITA, is more sample-efficient than single-stage MLE fine-tuning. PITA’s approach not only improves task performance but also reduces the need for large amounts of labeled preference data, making it more computationally efficient. Furthermore, PITA's ability to maintain alignment with the reference policy while improving accuracy highlights its robustness in real-world applications. Additional insights are presented in Appendix~\ref{app:experiments}.

\subsection{Performance on Star-Graph Reasoning}

To better understand the effectiveness of our method in guiding a reference model at test time, we analyze the behavior of \textsc{PITA} alongside other value-guided algorithms on star-graph reasoning tasks~\citep{bachmann2024pitfalls}.  

A star-graph \(\mathcal{G}(d, l)\) is defined as a tree with \(d\) paths of length \(l\) emanating from a common root node. The task is for the language model (LM) to generate the correct path from the root to a leaf, given the edge set defining the graph. Figure~\ref{fig:star-graph} illustrates the star-graph task.

\begin{figure}[ht]
  \centering
  \includegraphics[width=0.9\textwidth]{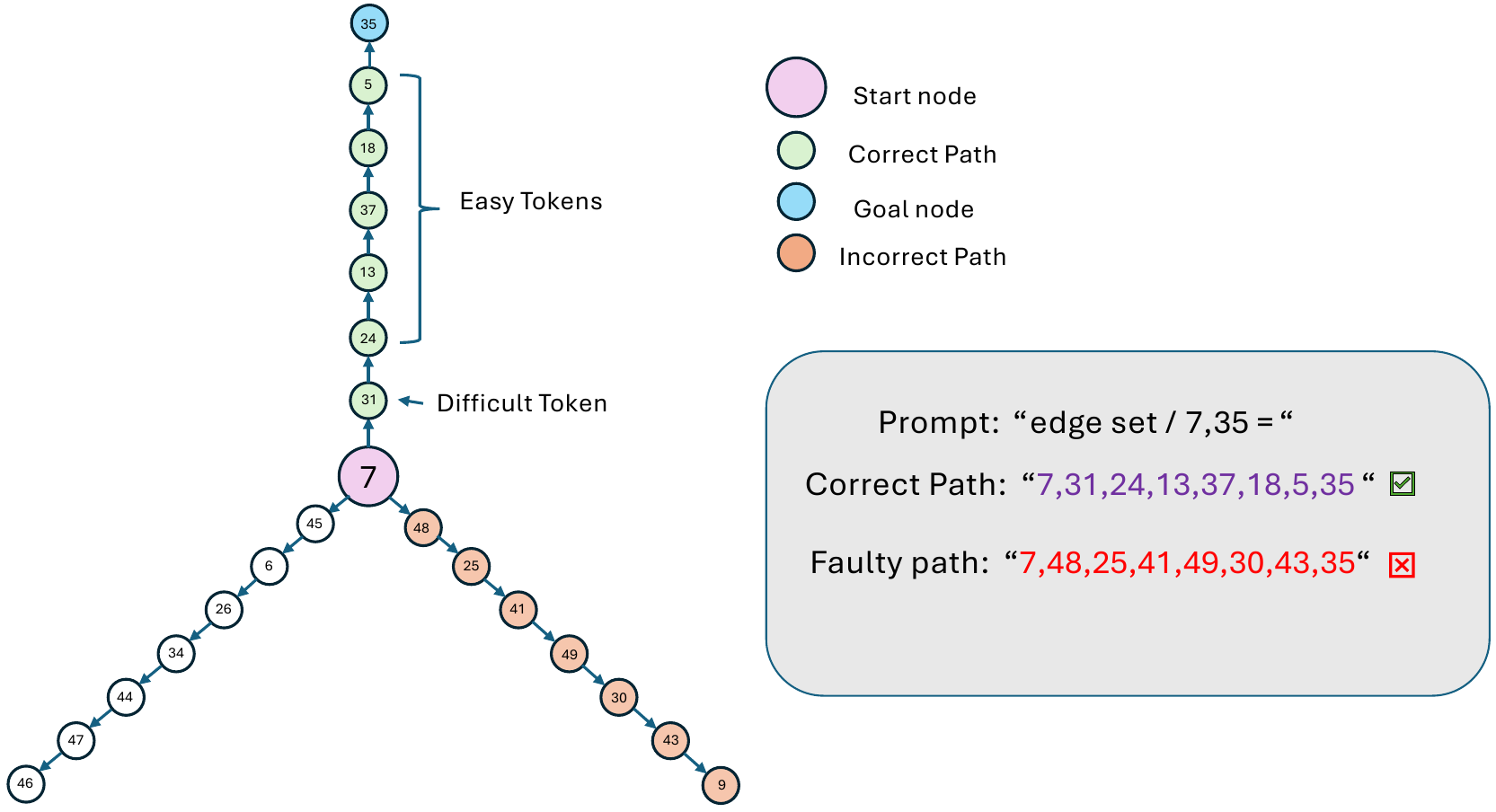}
  \caption{A \(\mathcal{G}= (3,8)\) star-graph configuration with degree \(d=3\) and path length \(l=8\). The pre-trained model, using next-token prediction, learns a faulty shortcut: it selects a random first node and follows the path from there.}
  \label{fig:star-graph}
\end{figure}

In addition to post-training algorithms, we compare the performance of PITA with the following methods:
\begin{itemize}[leftmargin=1.4em]
    \item \textbf{CD}~\cite{mudgal2023controlled}: A controlled decoding algorithm trained with access to true binary reward labels.  
    \item \textbf{REINFORCE}~\citep{ahmadian2024back}: A post-training algorithm that uses REINFORCE instead of the canonical Proximal Policy Optimization (PPO) in reinforcement learning from human feedback.  
    \item \textbf{DPO}~\citep{rafailov2023direct}: A post-training algorithm that aligns language models with human preferences by optimizing for preferred responses over less preferred ones, without using reinforcement learning.
    \item \textbf{RPO}~\citep{pang2024iterative}: A post-training algorithm that modifies the DPO loss to improve performance on reasoning tasks.
\end{itemize}

Despite the task’s structural simplicity, \citet{bachmann2024pitfalls} highlight a key failure mode in models trained via next-token prediction. Specifically, the model often learns a \textit{Clever-Hans} shortcut—selecting the first node in the root’s adjacency list and following it to a leaf. While this heuristic yields perfect accuracy on the training set, it fails to generalize: on unseen test graphs, the model selects the correct path only with probability \(1/d\), as the first adjacent node is correct only \(1/d\) of the time.

\citet{zhou2025q} show that this faulty behavior can be corrected during post-training using value-based methods. However, widely used policy-based approaches such as REINFORCE, DPO, and RPO fail to overcome this shortcut, achieving test accuracy no better than \(1/d\). This limitation is likely due to the difficulty transformer-based architectures face in unlearning spurious heuristics via policy-gradient optimization~\citep{hu2024learning}. 

\textbf{Main Results:} In Table~\ref{tab:pita-star-graph}, we report the test accuracies of PITA on three classes of graphs: smaller graphs \(\mathcal{G}(2, 5)\) and \(\mathcal{G}(5, 5)\), where we use the GPT-2 small model, and a larger graph \(\mathcal{G}(3, 8)\), where we use the GPT-2 medium model~\cite{radford2019language} to train the set of pre-trained models. We observe that \textbf{PITA} effectively guides the reference LM to achieve near-perfect generalization accuracy on the star-graph task, outperforming strong value-based baselines such as \textsc{Q\#} and \textsc{CD}, which themselves achieve high generalization performance.

\begin{table}[H]
    \centering
    \caption{Comparison of PITA with other baselines across star-graph configurations.}
    \label{tab:pita-star-graph}
    \begingroup
    \setlength{\tabcolsep}{6pt}
    \renewcommand{\arraystretch}{1.2}
    \begin{tabular}{lccc}
        \toprule
        \textbf{Algorithm} & $\mathcal{G}(2,5)$ & $\mathcal{G}(5,5)$ & $\mathcal{G}(3,8)$ \\
        \midrule
        pre-trained & 49.94 & 19.8 & 33.24 \\
        Q\# & 99.90 & 97.00 & 98.62 \\
        CD & 99.94 & 98.62 & 99.52 \\
        REINFORCE & 49.19 & 13.13 & 25.96 \\
        DPO & 37.61 & 0.00 & 0.00 \\
        RPO & 49.69 & 19.93 & 33.47 \\
        \midrule
        \textbf{PITA (ours)} & 99.97 & 99.70 & 99.93 \\
        \bottomrule
    \end{tabular}
    \endgroup
\end{table}

The star-graph experiments provide insight into why PITA outperforms standard policy-based methods, as well as value-based algorithms like Q\# and CD. While REINFORCE achieves a maximum accuracy of \(1/d\), DPO performs even worse. We observe behavior similar to that reported in~\cite{zhou2025q}, where DPO achieves 0 accuracy due to policy collapse (i.e., assigning zero probability mass). In contrast, value-based methods perform well across all graph configurations. Notably, PITA further improves accuracy, even when baseline value-based methods already achieve near-perfect performance.

This improvement can be attributed to PITA's ability to generate more precise implicit rewards. Upon examining the rewards generated by Q\# on the star-graph reasoning test set, we find that the reward differences between correct and incorrect answers are significantly smaller than those produced by PITA. Additional training details are provided in Appendix~\ref{app:experiments}.



\section{Conclusion}

In this work, we presented \textbf{PITA: Preference-Guided Inference-Time Alignment}, a novel framework for aligning large language models (LLMs) with user preferences during inference, without the need for extensive fine-tuning or pre-trained reward models. By directly integrating preference feedback into the decoding process, PITA offers a data-efficient, computationally lighter alternative to traditional alignment methods. We demonstrated its effectiveness across diverse tasks, such as mathematical reasoning and sentiment classification, achieving performance comparable to reward-based alignment methods, while outperforming those that require learning a reward model from preference data.  Limitations to our approach include the need for high-quality preference data, and computational demands of the iterative nature of the refinement process.   Additionally, while PITA eliminates the need for pre-trained reward models, it still relies on a reference policy.

\clearpage

\clearpage
\appendix
\setcounter{page}{1}

\begin{tcolorbox}[colback=gray!5!white, colframe=gray!75!black, 
                  title=Code Submission, arc=4mm, boxrule=0.4pt]
Our code is available at \url{https://github.com/SaratBobbili/pita}
\end{tcolorbox}

\section{Additional Related Work}\label{app:related-work-extd}

We provide an expanded discussion of related work omitted from the main paper.

\paragraph{Alignment:}
Large language models (LLMs) have demonstrated emerging capabilities in advanced natural language tasks through the use of \textbf{attention mechanisms} and the \textbf{transformer architecture}~\citep{vaswani2017attention,radford2019language,brown2020language,devlin2018bert,bubeck2023sparks}. In the context of LLMs, the goal is to fine-tune the responses generated for downstream tasks. \textbf{Reinforcement Learning from Human Feedback (RLHF)}~\citep{christiano2017deep,ziegler2019fine,stiennon2020learning,ouyang2022training} has become a significant technique in aligning LLMs with human values and preferences. A key aspect of the RLHF approach is the use of a \textbf{ground-truth reward model} (either a pre-trained model from preference data or labels from supervised data).

An alternative approach for utilizing preference information is \textbf{Direct Preference Optimization (DPO)}~\citep{rafailov2023direct}, which updates the model's policy directly using preference data. Unlike RLHF, DPO avoids reliance on a separate reward model by learning from relative preferences between pairs of outputs. This approach enables more direct control over model behavior, without the complexities and instability often associated with reward-based methods. The approach has been extended to relate DPO to learning an optimal soft Q-function~\citep{rafailov2024r}, as well as more general preference functions~\citep{azar2024general}.

\paragraph{LLM Reasoning:}
Recent methods further extend the capabilities of LLMs to generate intermediate reasoning steps. The standard approach in LLM-based reasoning begins with an initial policy $\pi$ instantiated by a reference LLM. The output is generated as a sequence of steps by auto-regressively predicting the next step using prompts. This idea is widely known through the \textbf{Chain-of-Thought (CoT)}~\citep{wei2022chain} approach to reasoning tasks. \textbf{Self-Consistency}~\citep{wang2022self} selects the most frequent answer from multiple generations. \textbf{Tree-of-Thought (ToT)}~\citep{yao2023tree} extends the CoT method by generating tree-based reasoning traces, encouraging exploration among various possible thoughts.

Another approach uses a \textbf{value function} to guide the selection of reasoning traces from a combinatorially large set. Two common approaches include \textbf{Outcome Reward Models (ORMs)}~\citep{cobbe2021training} and \textbf{Process Reward Models (PRMs)}~\citep{li2024process}. While ORMs are trained only on the accuracy of the final answer, PRMs are trained on the accuracy of intermediate reasoning steps. A popular approach, \textbf{Best-of-N}~\citep{lightman2023let}, combines the utility of a value function (either PRM or ORM) by selecting the reasoning trace with the highest value.

\paragraph{Inference-Time Scaling:}
Inference-time compute methods aim to enhance the reasoning capabilities of LLMs by allowing the models to `think' longer or perform additional computation during inference. ~\citep{snell2024scaling} provides a unified perspective on test-time computation as modifying an LLM's distribution. Specifically, we focus on methods that modify outputs using \textbf{post-hoc scorers}, typically employed in \textbf{Markov Chain Monte Carlo (MCMC) algorithms}~\citep{andrieu2003introduction}.

\textbf{ReST}~\citep{singh2023beyond} is a self-training method where the model generates samples filtered by binary feedback, fine-tunes on these filtered samples, and repeats the process. It outperforms human-only fine-tuning and reduces reliance on human data. \textbf{STaR (Self-Taught Reasoner)}~\citep{zelikman2024star} generates rationales using a few examples, retries with the correct answer if the output is wrong, fine-tunes on successful rationales, and repeats the process. This enhances performance on reasoning tasks without large rationale datasets, rivaling much larger models.

\textbf{MCTS-based approaches} and similar planning strategies~\citep{liu2023don,zhang2023planning,zhou2023language,choi2023kcts} that leverage a learned value network have demonstrated powerful test-time scaling by guiding inference with value estimates, resulting in more coherent and preferable text generation. Tree-based methods, like \textbf{Tree-of-Thought}~\citep{yao2023tree} and \textbf{RAP}~\citep{hao2023reasoning}, enhance LLM reasoning with shallow tree-search but struggle with long-horizon tasks or limited model knowledge.

\textbf{TS-LLM}~\citep{feng2023alphazero} addresses these issues using an AlphaZero-inspired framework with a learned value function, enabling deeper, more adaptable reasoning and improving LLM performance during both inference and training. \textbf{ReST-MCTS*}~\citep{zhang2024rest} integrates enhancements from ReST and TS-LLM by using tree-search reinforcement learning to infer per-step rewards, enabling higher-quality reasoning trace collection and iterative policy training, outperforming prior methods like Best-of-N.

\textbf{TreeBoN}~\citep{qiu2024treebon} improves upon Best-of-N sampling by using speculative tree-search and token-level DPO rewards to prune low-quality paths, achieving better output quality with lower computational cost and outperforming standard Best-of-N across multiple benchmarks.

\textbf{s1}~\citep{muennighoff2025s1} is a simple test-time scaling method that improves model reasoning by using a curated reasoning dataset (s1K) and a technique called \textbf{budget forcing} to control its thinking length. This method forcefully terminates the model's output or extends it by appending a “Wait” token multiple times when the model attempts to end.

\section{Proofs of Base Results}\label{app:theory}
In this section, we provide detailed proofs of the main and auxiliary results needed to derive the structure of the optimal function $Q^{*,\eta} , V^{*,\eta}$.

\subsection{Notational Preliminaries}
We now introduce several notational conventions used frequently throughout the article. For clarity and conciseness, we typically adopt the most compact notation possible, often omitting information when the meaning is clear from context.

Given a set $\cA$, for any two elements $x,y \in \cA$, we denote that $x$ is preferred over $y$ by $x \succ y$. Therefore, the indicator function that denotes preference between two elements is $\mathbf{1}_{x \succ y}$. Recall we denote the generation starting from a given context by $y^{\pi}_x$ according to a policy $\pi$. We interchangeably use $\pi$ to denote the policy, and the LLM model as well

\textbf{Greedy Decoding (GD):}
Let $\pi$ be a language model (LLM) policy, a fixed context $x$, and a target generation length $T$. The greedy decoding strategy produces an output sequence $y_{\pi;\text{GD}} = (x, a_{1:T})$, where each token $a_t$ is deterministically selected according to
\begin{equation}
a_t = \arg\max \pi(\cdot \mid x, a_{1:t-1}).
\end{equation}

We use $x \in \R^d$ to denote a $d$-dimensional vector, $M \in R^{m \times n}$ denotes a matrix of dimension $m \times n$.
The semi-norm is defined as \( \|u\|_{\Sigma_{\cD} + \lambda I} \triangleq \sqrt{u^\top (\Sigma_{\cD} + \lambda I) u} \). The $d$-dimensional all-ones vectors is defined as follows
\[
\mathbf{1} = \left[ 1, 1, \cdots, 1 \right]^{\top} \in \mathbb{R}^d
\]

The logistic function is denoted by $\sigma(x) \triangleq \frac{1}{1 + e^{-x}}$, and its inverse by $\Psi(t) \triangleq \frac{t}{(1-t)}$. 

\subsection{Optimal Soft Q and V Functions}

 We provide the proof of optimality for the policy structure considered in equations~\ref{eqns:soft-bellman} for the KL-regularized optimization objective~\ref{eqn:klreg-soft-obj}.  This is a well-established result and we only provide it for completeness. With a slight abuse of notation, we denote the optimal policy by $\pi_t^*$, and the corresponding optimal soft-value functions by $V_t^*,Q_t^*$

The optimal solution to the KL-regularized reward maximization objective in ~\ref{eqn:kl-obj} ~\cite{rafailov2023direct} is given by
\begin{equation}\label{eqn:app-optimal-reward}
\pi^*_r(y \mid x) \propto  \pi_{\text{ref}}(y \mid x) \exp\left( \frac{1}{\eta} r(x, y) \right)
\end{equation}
We can make the following inductive argument for the optimality of the soft-value functions $V_t^{\star,\eta},Q_t^{\star,\eta}$ starting at at given state $s_t = s$ and action $a_t = a$.\\

\textbf{Base case}: let $s_{T-1} = s$. Since $V_{T}^{\star,\eta}(s) = 0$, we have by definition $Q_{T-1}^{\star,\eta}(s,a) = r(s,a)$ .  Therefore, we have from Equation~\ref{eqn:app-optimal-reward} the expression for the optimal policy, and subsequently the soft-value function as  
\begin{align*}
     \pi^*_{T-1}(a \mid s) &\propto  \pi_{\text{ref}}(a \mid s) \exp\left( \frac{1}{\eta} Q_{T-1}^{\star,\eta}(s,a) \right) \propto \pi^{*,\eta}_{T-1}(a \mid s)\\
     \pi^*_{T-1}(a \mid s) &= \pi^{*,\eta}_{T-1}(a \mid s) = \frac{\pi_{\text{ref}}(a \mid s) \exp\left( \eta^{-1} Q_{T-1}^{\star,\eta}(s,a) \right)}{\mathbb{E}_{ \pi_{\mathrm{ref}}} \left[ \exp\left( \eta^{-1} Q_{T-1}^{\star,\eta}(s,a) \right) \right]} \\
     V_{T-1}^{\star,\eta}(s) &= \mathbb{E}_{\pi^{*,\eta}} \left[  r_{T-1}(s,a) - \eta \, \mathrm{KL}(\pi^{*,\eta}(s_{T-1}) \parallel \pi_{\mathrm{ref}}(s_{T-1})) \Big| s_{T-1} = s \right]\\
     &= \mathbb{E}_{\pi^*} \left[  Q_{T-1}^{\star,\eta}(s,a) - \eta \, \mathrm{KL}(\pi^*(s_{T-1}) \parallel \pi_{\mathrm{ref}}(s_{T-1})) \Big| s_{T-1} = s \right]\\
     &=\mathbb{E}_{\pi^*} \left[ Q_{T-1}^{\star,\eta}(s,a) \right] - \eta \left[ \mathrm{KL}(\pi^*(s) \parallel \pi_{\mathrm{ref}}(s)) \right]\\
     &=\mathbb{E}_{\pi^*} \left[ Q_{T-1}^{\star,\eta}(s,a) \right] - \eta \mathbb{E}_{\pi^*} \left[ \eta^{-1} Q_{T-1}^{\star,\eta}(s,a) \right] + \eta \ln \mathbb{E}_{ \pi_{\mathrm{ref}}} \left[ \exp\left( \eta^{-1} Q_{T-1}^{\star,\eta}(s,a) \right) \right]\\
     &=\eta \ln \mathbb{E}_{ \pi_{\mathrm{ref}}} \left[ \exp\left( \eta^{-1} Q_{T-1}^{\star,\eta}(s,a) \right) \right]
\end{align*}
\textbf{Induction Case:} Let's suppose $\pi_t^{\star,\eta}$ is optimal from horizon $h+1$ until $T-1$.
\begin{align*}
     Q_{h}^{\star,\eta}(s,a) &=  r_h(s,a)  + \mathbb{E}_{s' \sim P(\cdot | s,a)}  \left[ V_{h+1}^{*,\eta}(s') \right]\\
     &=  r_h(s,a)  + \mathbb{E}_{s' \sim P(\cdot | s,a)}  \left[ V_{h+1}^{*}(s') \right] = Q_{h}^{\star}(s,a)\\
\end{align*}
Denoting the state at horizon $h$ by $s_h = s$, using a policy of the form $\pi = (\pi_h,\pi^{\star}_{h+1},\dots,\pi^{\star}_{T-1})$ we have
\begin{align*}
    V_{h}^{\pi}(s) &= \mathbb{E}_{\pi} \left[  \sum_{t \geq h} r_t(s,a) - \eta \, \mathrm{KL}(\pi(s_t) \parallel \pi_{\mathrm{ref}}(s_t)) \Big| s_{h} = s \right]\\
    &= \mathbb{E}_{\pi_h}  \left[ r_h(s,a) - \eta \, \mathrm{KL}(\pi_h(s_h) \parallel \pi_{\mathrm{ref}}(s_h))  + \mathbb{E}_{\pi^{*}}  \left[ \sum_{t > h} r_t(s,a) - \eta \, \mathrm{KL}(\pi^{*}(s_t) \parallel \pi_{\mathrm{ref}}(s_t)) \Big| s_{h} = s \right] \right]\\
    &= \mathbb{E}_{\pi_h}  \left[ r_h(s,a) - \eta \, \mathrm{KL}(\pi_h(s_h) \parallel \pi_{\mathrm{ref}}(s_h))  + \mathbb{E}_{s' \sim P(\cdot | s,a)}  \left[ V_{h+1}^*(s') \right] \right]\\
    &= \mathbb{E}_{\pi_h}  \left[ Q_{h}^{\star,\eta}(s,a) - \eta \, \mathrm{KL}(\pi_h(s_h) \parallel \pi_{\mathrm{ref}}(s_h)) \right]
\end{align*}

Using the results of ~\cite{rafailov2023direct} by setting $r(s,a) = Q_{h}^{\star}(s,a)$,
\begin{align*} 
   V_{h}^{\star}(s) &= \max_{\pi_h} \mathbb{E}_{\pi_h}  \left[ Q_{h}^{\star}(s,a) - \eta \, \mathrm{KL}(\pi(s_h) \parallel \pi_{\mathrm{ref}}(s_h)) \right] = V_{h}^{\star,\eta}(s)
\end{align*}

For deterministic MDP's, the expression for the optimal soft-value expressions has a simple form in terms of the cumulative rewards from rollouts according to $\pi_{\mathrm{ref}}$~\cite{zhou2025q} as shown below.
\begin{align*}
\exp(\eta^{-1} V_{t}^{\star, \eta}(s)) 
&= \mathbb{E}_{\pi^{\text{ref}}} \left[ \exp\left( \eta^{-1} r_t(s_t, a_t) + \eta^{-1} V_{t+1}^{\star, \eta}(s_{t+1}) \right) \right]  \\
&= \mathbb{E}_{\pi^{\text{ref}}} \left[ \exp\left( \eta^{-1} \sum_{t} r_t(s_t, a_t) \right) \, \middle| \, s_t = s \right] 
\end{align*}
Taking logarithm on both sides, we get
\begin{align*}
V_{t}^{\star, \eta}(s) 
&= \eta \ln \mathbb{E}_{\pi^{\text{ref}}} \left[ \exp\left( \eta^{-1} \sum_{t} r_t(s_t, a_t) \right) \, \middle| \, s_t = s \right] 
\end{align*}
The result for $Q_t^{\star,\eta}$ follows a similar line of analysis.


\subsection{Proof of Lemma 1}
\begin{proof}
    We use a result on the equivalence class induced by two reward functions under the same preference distribution from~\citet{rafailov2023direct}. For any given context $x$, and policy $\pi$, let $y_x^{\pi}$ be the unique greedy completion under policy $\pi$. Setting $\tilde{r}(y_x) \triangleq r^*(y_x) - r^*(y_x^{\pi})$, we observe that $\tilde{r}(y_x) - r^*(y_x)$ is entirely a functional of $x$. It follows from ~\citet{rafailov2023direct}[Lemma 1] that the preference distribution induced by $\tilde{r}(\cdot)$ is identical to $\cP^*$.
\end{proof}

\subsection{Proof of Theorem 1}
\begin{proof}
The inverse function of the logistic function $\sigma(x)$ is given by $\Psi(x) \triangleq \log{\frac{x}{(1-x)}}$. Suppose that the Bradley-Terry model holds.  For any $y,y'$ given a context $x$, we have:
\begin{align*}
    \Psi \left(\mathcal{P}^*(y \succ y' | x)\right) &= \Psi(\sigma(r(y)-r(y')))\\
    &= r(y) -  r(y')\\
   r(y) &= r(y') + \Psi \left(\mathcal{P}^*(y \succ y' | x)\right).
\end{align*}
Setting $y := s_T$, and $y' := y^{\mathrm{ref}}_s$, the above equivalence of the reward function to the preference distribution yields the following expressions for the optimal value function $V_{t}^{\star, \eta}$
\begin{align}
  V_{t}^{\star, \eta}(s) &= \eta \ln \mathbb{E}_{\pi_{\text{ref}}} \left[ \exp \left( \eta^{-1}  r(s_{T-1}, a_{T-1}) \right) \middle| s_t = s \right]\\
  &=  \eta \ln \mathbb{E}_{y_s \sim \pi_{\text{ref}}} \left[ \exp \left( \eta^{-1}  \left[ r(y^{\mathrm{ref}}_s) + \Psi \left(\mathcal{P}^*(y_s \succ y^{\mathrm{ref}}_s )\right) \right] \right) \middle| s_t = s \right], \\
  &=  r(y^{\mathrm{ref}}_s) + \eta \ln \mathbb{E}_{y_s \sim \pi_{\text{ref}}} \left[ \exp \left( \eta^{-1}  \left[\Psi \left(\mathcal{P}^*(y_s \succ y^{\mathrm{ref}}_s )\right) \right] \right) \middle| s_t = s \right], \\
\end{align}
Similarly, the expression of $Q_{t}^{\star, \eta}$ can be written in terms of the preference distribution.
\end{proof}

\section{Convergence of PITA} \label{app:convergence-proofs}

\subsection{Bounding the MLE estimation error}

We present a proof sketch of the MLE estimation error bound in Lemma~\ref{eqn:mle}, which closely follows~\cite{zhu2023principled} [Lemma 3.1]. We first establish the strong convexity of the loss function, and then bound the estimation error to obtain the final result.

\begin{definition}
 (Strong Convexity) ~\citep{beck2017first}[Theorem 5.8] A function \( f : \mathbb{E} \rightarrow (-\infty, \infty] \) is called \( \sigma \)-strongly convex for a given \( \sigma > 0 \) if \( \mathrm{dom}(f) \) is convex and the following inequality holds for any \( \mathbf{x}, \mathbf{y} \in \mathrm{dom}(f) \) and \( \lambda \in [0, 1] \):


\[
f(\mathbf{y}) \geq f(\mathbf{x}) + \langle \mathbf{g}, \mathbf{y} - \mathbf{x} \rangle + \frac{\sigma}{2} \|\mathbf{x} - \mathbf{y}\|^2 
\quad \text{for any } \mathbf{x} \in \mathrm{dom}(\partial f),\, \mathbf{y} \in \mathrm{dom}(f) 
\text{ and } \mathbf{g} \in \partial f(\mathbf{x}),
\]
where $\partial f(\mathbf{x})$ is the sub-differential of $f$ at $\mathbf{x}$.
\end{definition}

\begin{lemma}[Strong Convexity of the MLE Loss]\label{lem:app-mle-cvx}

The MLE loss $\ell_{\cD}$ is strongly convex at $\theta^\star$ with respect to the semi-norm $\|\cdot\|_{\Sigma_{\cD}}$. Specifically, there exists a constant $\gamma > 0$ such that
\begin{equation}
    \ell_{\cD}(\theta^\star + \Delta) - \ell_{\cD}(\theta^\star) - \langle \nabla \ell_{\cD}(\theta^\star), \Delta \rangle \geq \gamma \|\Delta\|^2_{\Sigma_{\cD}}
\end{equation}
for all perturbations $\Delta \in \mathbb{R}^d$ such that $\theta^\star + \Delta \in \Theta_B$.
\end{lemma}

\begin{proof}
For simplicity, we define \( x_i \triangleq \Phi(y_i) \). The explicit expression for $\ell_{\cD}$ is given by:
\begin{align}
\ell_{\mathcal{D}}(\theta) &= - \sum_{i=1}^n o_i \log{\sigma(\langle \theta, x_i \rangle)} + 
(1 - o_i) \cdot \log{(1 - \sigma(\langle \theta, x_i \rangle))}, \\
\end{align}
where $o_i = \mathbf{1}_{y_i \succ y_i^{\mathrm{ref}}}$.

Next, we compute the Hessian of $\ell$:
\begin{align*}
    \nabla^2 \ell_{\cD}(\theta) &= \frac{1}{n} \sum_{i=1}^n \left( o_i \cdot \frac{\exp(-\langle \theta, x_i \rangle)}{(\exp(-\langle \theta, x_i \rangle)+1)^2} + (1 - o_i) \cdot \frac{\exp(\langle \theta, x_i \rangle)}{(\exp(\langle \theta, x_i \rangle)+1)^2} \right) \cdot x_i x_i^\top \\
    &= \frac{1}{n} \sum_{i=1}^n \frac{\exp(-\langle \theta, x_i \rangle)}{(\exp(-\langle \theta, x_i \rangle)+1)^2} \cdot x_i x_i^\top.
\end{align*}

From Assumption~\ref{assum:theta_bound}, we have that $\langle \theta, x_i \rangle \in [-2LB, 2LB]$. This gives the following lower bound:
\[
    \frac{\exp(-\langle \theta, x_i \rangle)}{(\exp(-\langle \theta, x_i \rangle)+1)^2} \geq \frac{1}{2 + \exp(-2LB) + \exp(2LB)}.
\]

Using this result and the expression for the Hessian above, we find:
\[
    z^\top \nabla^2 \ell_{\cD}(\theta) z \geq \frac{\gamma}{n} \|Xz\|_2^2 \quad \text{for all } z,
\]
where $\gamma \triangleq \frac{1}{2 + \exp(-2LB) + \exp(2LB)}$ and $X \in \mathbb{R}^{n \times d}$, with the $i^{\text{th}}$ row of $X$ given by $x_i \in \mathbb{R}^d$.

Finally, using the second-order mean value theorem, we obtain:
\[
\ell_{\cD}(\theta^\star + \Delta) - \ell_{\cD}(\theta^\star) - \langle \nabla \ell_{\cD}(\theta^\star), \Delta \rangle \geq \gamma \|\Delta\|_{\Sigma_{\cD}}^2,
\]
where the semi-norm is defined as \( \|u\|_{\Sigma_{\cD} + \lambda I} \triangleq \sqrt{u^\top (\Sigma_{\cD} + \lambda I) u} \).
\end{proof}

\begin{proof}[\textbf{Proof sketch of Lemma~\ref{eqn:mle}}]

By definition $\hat{\theta}_{\text{MLE}}$ is optimal for $\ell_D$, hence $\ell_D(\hat{\theta}_{\text{MLE}}) \leq \ell_D(\theta^\star)$. Setting the error vector $\Delta = \hat{\theta}_{\text{MLE}} - \theta^\star$, from the $\gamma$-convexity condition, and Holder's inequality on the norm $\|\cdot\|_{\Sigma_{\mathcal{D}} + \lambda I}$ it follows that
\begin{align*}
\gamma \|\Delta\|^2_{\Sigma_D} &\leq \ell_D(\theta^\star + \Delta) - \ell_D(\theta^\star) - \langle \nabla \ell_D(\theta^\star), \Delta \rangle\\
&\leq -\langle \nabla \ell_D(\theta^\star), \Delta \rangle \leq \|\nabla \ell_D(\theta^\star)\|_{(\Sigma_D + \lambda I)^{-1}} \|\Delta\|_{\Sigma_D + \lambda I}.
\end{align*}

 

With probability at least $(1-\delta)$, the term $\|\nabla \ell_D(\theta^\star)\|_{(\Sigma_D + \lambda I)^{-1}}$ is further bounded using tail concentration inequalities~\cite{hsu2012tail} as follows

\[
\|\nabla \ell_{\mathcal{D}}(\theta^*)\|^2_{(\Sigma_{\mathcal{D}} + \lambda I)^{-1}} \leq \Tilde{C} \cdot \frac{d + \log(1/\delta)}{n}.
\]

for some universal constant $C$. Assumption~\ref{assum:theta_bound}, combined with previous results gives us
\begin{align}
\gamma \|\Delta\|^2_{\Sigma_{\mathcal{D}} + \lambda I} 
&\leq \|\nabla \ell_{\mathcal{D}}(\theta^*)\|_{(\Sigma_{\mathcal{D}} + \lambda I)^{-1}} \|\Delta\|_{\Sigma_{\mathcal{D}} + \lambda I} + 4 \lambda \gamma B^2\\
&\leq \sqrt{\Tilde{C} \cdot \frac{d + \log(1/\delta)}{n}} \|\Delta\|_{\Sigma_{\mathcal{D}} + \lambda I} + 4 \lambda \gamma B^2.\\
\|\Delta\|_{\Sigma_{\mathcal{D}} + \lambda I} &\leq C \cdot \sqrt{ \frac{d + \log(1/\delta)}{\gamma^2 n}  + \lambda B^2}.
\end{align}
Where $C$ can be chosen to be an $o(1)$ constant which is bounded by a multiple of $\sqrt{\Tilde{C}}$. 

\end{proof}

\subsection{Proof of Theorem 2}



Before delving into the details, we provide a brief overview of the steps involved in deriving the regret bound on PITA. In contrast to the approach of~\cite{zhou2025q}, we do not directly rely on the Performance Difference Lemma. Instead, we leverage a closed-form expression for $V^{\theta,\eta}(\cdot)$ that explicitly depends on the parameter $\theta$. Our key insight is that the gradient of this parameterized value function with respect to $\theta$ can be expressed directly in terms of the policy space; see Equation~\eqref{EQ:GRADLN}.

Within this framework, Assumption~\ref{assum:emb-norm-bound} plays a crucial role by bounding the quantity $\E_{y_x \sim \pi_{\theta}(y_x)} \left[\eta^{-1}\Phi(y_x) \right]$ for any $\theta$. This condition is essential because it excludes deterministic policies and ensures a degree of exploration across all learned policies. Notably, this assumption aligns with standard practices in KL-regularized policy updates, where a similar condition is imposed with $\pi_{\theta}$ replaced by a reference policy $\pi_{\text{ref}}$. In our case, this reference policy corresponds to an initial estimate of $\theta$.

\begin{proof}
For a given iteration $k$, let $\theta_k = \hat{\theta}$ be the MLE estimate. Consider the following set of parameters:
\begin{equation}
\Theta(\hat{\theta}, \lambda) = \left\{ \theta \in \Theta_B \,\middle|\,
\left\|\hat{\theta} - \theta \right\|_{\Sigma_{\mathcal{D}} + \lambda I}
\leq C \cdot \sqrt{\frac{d + \log\left(\frac{1}{\delta}\right)}{\gamma^2 n} + \lambda B^2} \right\}.
\end{equation}
For a given context $x$, the difference between the optimal value function and the estimated value function is as follows:
\begin{align}
V^{\star,\eta}(x) - V^{\hat{\theta},\eta}(x) &= \eta \ln \mathbb{E}_{y_x \sim \pi_{\text{ref}}} \left[ \exp \left( \eta^{-1} \langle \theta^*, \Phi(y_x) \rangle \right) \middle| x \right] \\
&\quad - \eta \ln \mathbb{E}_{y_x \sim \pi_{\text{ref}}} \left[ \exp \left( \eta^{-1} \langle \hat{\theta}, \Phi(y_x) \rangle \right) \middle| x \right]
\end{align}
By the mean value theorem, we have for some $\bar{\theta}$ on the line joining $\theta^*$ and $\hat{\theta}$ that:
\begin{align}
&\left| \ln \mathbb{E}_{y_x \sim \pi_{\text{ref}}} \left[ \exp \left( \eta^{-1} \langle \theta^*, \Phi(y_x) \rangle \right) \middle| x \right] - \ln \mathbb{E}_{y_x \sim \pi_{\text{ref}}} \left[ \exp \left( \eta^{-1} \langle \hat{\theta}, \Phi(y_x) \rangle \right) \middle| x \right] \right|\\
&= \left| \langle \theta^* - \hat{\theta} , \nabla_{\theta} \ln \mathbb{E}_{y_x \sim \pi_{\text{ref}}} \left[ \exp \left( \eta^{-1} \langle \theta, \Phi(y_x) \rangle \right) \middle| x \right]|_{\theta = \bar{\theta}} \rangle \right|
\end{align}
From H\"older's inequality on the $\|\cdot\|_{\Sigma_{\mathcal{D}} + \lambda I}$ norm and its dual, we obtain:
\begin{align}
    &\left| \langle \theta^* - \hat{\theta} , \nabla_{\theta} \ln \mathbb{E}_{y_x \sim \pi_{\text{ref}}} \left[ \exp \left( \eta^{-1} \langle \theta, \Phi(y_x) \rangle \right) \middle| x \right]|_{\theta = \bar{\theta}} \rangle \right|\\
    &\leq  \|\hat{\theta} - \theta\|_{\Sigma_{\mathcal{D}} + \lambda I} \|\nabla_{\theta} \ln \mathbb{E}_{y_x \sim \pi_{\text{ref}}} \left[  \exp \left( \eta^{-1} \langle \theta, \Phi(y_x) \rangle \right) \middle| x \right]|_{\theta = \bar{\theta}} \|_{(\Sigma_{\mathcal{D}} + \lambda I)^{-1}}\\
    &= \left\|(\hat{\theta} - \theta) (\Sigma_{\mathcal{D}} + \lambda I)^{1/2} \right\|_2 \left\| (\Sigma_{\mathcal{D}} + \lambda I)^{-1/2} \nabla_{\theta} \ln \mathbb{E}_{y_x \sim \pi_{\text{ref}}} \left[  \exp \left( \eta^{-1} \langle \theta, \Phi(y_x) \rangle \right) \middle| x \right]|_{\theta = \bar{\theta}} \right\|_2\\
    &= \left\|(\hat{\theta} - \theta) \right\|_{\Sigma_{\mathcal{D}} + \lambda I} \left\| (\Sigma_{\mathcal{D}} + \lambda I)^{-1/2} \nabla_{\theta} \ln \mathbb{E}_{y_x \sim \pi_{\text{ref}}} \left[  \exp \left( \eta^{-1} \langle \theta, \Phi(y_x) \rangle \right) \middle| x \right]|_{\theta = \bar{\theta}} \right\|_2
\end{align}
The gradient of the log-expectation in the second term can be written as the expected embedding under policy $\pi_{\bar{\theta}} \in \Pi_{\theta}$. That is:
\begin{align}\label{EQ:GRADLN}
\nabla_{\theta} \ln \mathbb{E}_{y_x \sim \pi_{\text{ref}}} \left[ \exp \left( \eta^{-1} \langle \theta, \Phi(y_x) \rangle \right) \middle| x \right]_{\theta = \bar{\theta}} &= \frac{1}{\eta} \E_{y_x \sim \pi_{\bar{\theta}}(y_x)} \left[\Phi(y_x) \right]
\end{align}

Taking the expectation with respect to the distribution of context $x$ and summing over $K$ iterations, from Assumption~\ref{assum:emb-norm-bound}, we obtain the following result:
\begin{align}
\sum_{k=1}^K (V^{\star,\eta} - V^{\theta_k,\eta}) \leq C \sum_{k=1}^K \left( \sqrt{\frac{d + \log(1/\delta)}{\gamma^2 n_k} + \lambda B^2} \cdot \sup \left\|(\Sigma_{\mathcal{D}} + \lambda I)^{-1/2} \mathbb{E}_{y \sim \pi_{\mathrm{ref}}} [\Phi(y)] \right\|_2 \right).
\end{align}
\end{proof}

\section{Implementation Details and Sentiment Generation}\label{app:experiments}
In this section we describe training details for PITA on GSM8K, and star graph reasoning. Additionally, we present our results in detail for sentiment generation tasks~\cite{chaudhary2025riskaversefinetuninglargelanguage}


\subsection{Math Reasoning}\label{app:math}
\paragraph{Dataset.}  \textbf{GSM8K} contains \(7{,}500\) training and \(1{,}319\) test problems.  
We create a \(90\%/10\%\) split of the original training set for learning and validation, reserving the full test set for final reporting.

\paragraph{Models.}  We use the \textbf{Llama-3 8B-Instruct} model as our reference policy and the \textbf{Llama-3.2 1B-Instruct} model ~\cite{grattafiori2024llama} as the backbone for our Value function estimation. For Q\#-HF, the reward model trained from preference data is also a Llama-3.2 1B-Instruct. We use the OpenMathInstruct-2 model~\cite{toshniwal2024openmathinstruct} for preference data generation between two pairs of responses for it's reasonably high accuracy on arithmetic reasoning tasks such as GSM8K.




\textbf{Training details.} We collect pairwise samples for each question in the training set of GSM8K and label every sample either as \textit{'preferred'} (1) or \textit{'not preferred'} (0) based on the quality of final answer. Note that for arithmetic reasoning, there is an exact numeric answer for each question. A correct answer is automatically preferred over a wrong answer. If the samples from a given pair are either both correct, or both wrong a preference is assigned based by the OpenMathInstruct-2 model.
PITA is trained over 10 policy evaluation/optimization cycles (5 per round), using a learning rate of $2 \times 10^{-5}$ and a context length of 4096. The model is retrained from scratch at the start of each round, using data collected from both the current and all previous rounds. The baseline model is trained with identical hyperparameters. Each optimization run is performed on a single NVIDIA A100 GPU.

\paragraph{Evaluation protocol.}  
We report single-sample accuracy (\textit{pass@1}) and majority-vote accuracy over \(k{=}\!8\) samples (\textit{maj1@8}).  
Generations use temperature \(T=0.8\) and nucleus sampling \(p=0.9\). We use the inference optimizations for value function $V$ given in ~\cite{zhou2025q}. Similar to Q\# we train PITA for two iterations and observe performance converging. All evaluations are done with zero-shot prompting.

The prompt template used for evaluation is:
\texttt{'Problem:\textbackslash n\textbackslash n\{0\} Write your answer inside \textbackslash\textbackslash boxed\{\{\}\}.\textbackslash n\textbackslash nSolution:'}
where \texttt{\{0\}} is replaced by the actual question from the dataset.

\subsection{Star-Graph Reasoning}
We now describe the training methodology used to apply \textbf{PITA} to the star-graph task.

\textbf{Training Setup:} We follow the experimental setup of Q\#~\cite{zhou2025q} to generate and evaluate the star-graph datasets. All models are initially pretrained using next-token prediction on a dataset of 200k randomly sampled graphs paired with their correct paths.

We consider several baselines, including \textbf{REINFORCE}, \textbf{DPO}, \textbf{RPO}—as well as value-based methods \textbf{Q\#} and \textbf{CD}. Post-training is conducted on a separate dataset of 200k new random graphs. To ensure a fair comparison, we reproduce the results with an identical reward structure and reuse the datasets used in~\cite{zhou2025q}. For REINFORCE, a reward of 1 is assigned for correct paths and \(-0.1\) for incorrect ones. For DPO and RPO, pairwise comparisons between a correct path \(y_{\text{chosen}}\) and an incorrect shortcut path \(y_{\text{reject}}\) are sampled from the pre-trained model. Similarly, for \textbf{Q\#}, a classifier is trained on the same pairwise dataset, labeling correct paths with reward 1 and incorrect ones with 0. Note that PITA uses a unique reference generation at the core of its algorithm. In PITA, we treat correct paths as the \textit{`preferred'} solution and the incorrect path as the \textit{`not-preferred'} answer to estimate the preference distribution. The unique correct answer naturally serves as the reference completion for the star-graph task. 

All models are trained for 10 epochs using the AdamW optimizer with a weight decay of $0.1$ and a batch size of $256$. The learning rates are set to $2.5 \times 10^{-4}$ for pretraining, $1 \times 10^{-5}$ for REINFORCE, and $1 \times 10^{-4}$ for DPO, RPO, CD, and Q\#. For Q\# and CD, we set the inverse temperature parameter $\eta = 0.1$. Training is performed on a single A100 GPU. Evaluation is conducted on a held-out test set of 20k graphs using top-$k = 10$ sampling and temperature $1.0$.

\subsection{IMDB Sentiment Generation}\label{app:imdb}

In this section we evaluate \textbf{PITA} on controlled sentiment generation. The task follows the IMDB-Gen setting of ~\cite{ramamurthy2023reinforcementlearningnotnatural} and ~\cite{chaudhary2025riskaversefinetuninglargelanguage}. Given the first 60 tokens of a movie review, the language model must continue the text so that the completed review expresses maximally positive sentiment.

\paragraph{Dataset.} We use the \textbf{IMDB Reviews} dataset, which contains \(25{,}000\) training and \(25{,}000\) test reviews, each labeled \emph{positive} or \emph{negative}. For each example, we truncate the review to its first 60 tokens, which form the prompt \(x\); the model then generates up to 60 continuation tokens \(y\). The reward model for this task is the \textbf{lvwerra/distilbert-imdb} classifier, with the logit corresponding to the positive sentiment output node used as the scalar reward. For the creation of the preference dataset, we select the generation with the higher reward as the preferred generation.

\paragraph{Models.} For this task, we use \textbf{TinyLlama-1.1B-Chat}\footnote{\url{https://huggingface.co/TinyLlama/TinyLlama-1.1B-Chat}} as both our reference policy and the backbone for our value function. For the reward model learned from the preference data for Q\#-HF, we use a \textbf{LLama 3.2 1B-Instruct} model.

\paragraph{Evaluation} We evaluate the same set of algorithms as above. For each of the \(5{,}000\) \textbf{IMDB} test set prompts, we sample a \emph{single} continuation using nucleus sampling with top-\(p = 0.9\) and temperature \(T = 0.8\). The full prompt + completion is scored by the \texttt{lvwerra/distilbert-imdb} sentiment classifier. We report the mean sentiment logit score (\textsc{Sent.}↑) and the token-level KL divergence (\textsc{KL}↓) between the fine-tuned policy and the reference model \(\pi_{\text{ref}}\).  An example of the completions generated by $\pi_{\text{ref}}$ and by PITA is shown in Figure~\ref{fig:imdb_example}.

\begin{table}[H]
  \centering
  \caption{IMDB-Gen: sentiment (higher is better) and KL divergence (lower is better) on 5,000 test prompts.}
  \label{tab:pita-imdb}
  \begingroup
  \setlength{\tabcolsep}{8pt}
  \renewcommand{\arraystretch}{1.2}
  \begin{tabular}{lcc}
    \toprule
    \textbf{Algorithm} & \textsc{Sent.}↑ & \textsc{KL}↓ \\
    \midrule
    $\pi_{\text{ref}}$ & -0.52 & - \\
    Q\#                & 0.22 & 7.51 \\
    Q\#-HF             & -0.44 & 5.22 \\
    \midrule
    \textbf{PITA (ours)} & 0.01 & 7.27 \\
    \bottomrule
  \end{tabular}
  \endgroup
\end{table}

\paragraph{Discussion.} The range of rewards generated by \texttt{lvwerra/distilbert-imdb} typically lies between -3 and +3. Q\# with access to true rewards performs the best. Among the algorithms that only have access to a preference signal, rather than the true reward, PITA performs the best. 

Sentiment alignment is inherently difficult, and Q\# has access to a pre-trained DistilBERT model specifically tuned for sentiment analysis. In contrast, reproducing a preference function with a similar structure requires much more preference data. Despite this challenge, PITA is able to extract valuable information from the given preference data and significantly outperforms Q\#-HF, which appears to have poorly learned a reward model. This demonstrates PITA's ability to leverage limited preference data for effective alignment.

\begin{figure}[ht]
\centering
\begin{minipage}{\textwidth}
\begin{tcolorbox}[colback=white, colframe=black]

\textbf{Prompt:} Paul Mazursky misfires on this film. The writing, direction, casting, and acting (with the exception of Victorio Gassman) are all off the mark. I remember the reviews from 20+ years ago being mediocre, but I thought it still might

\dasheddline

\textbf{Completion from $\pi^{\text{ref}}$:} \textcolor{red!80!black}{be good. But by today's standards, the film just isn't good enough. It's a muddled and overly long (about 2 1/2 hours) film that ultimately does not live up to its potential as a classic film noir.
}\\ 
\dasheddline

\textbf{Completion from PITA:} \textcolor{blue!30!green}{be good. I'm glad to see how far it's come. This is not a great film, but it's definitely a lot better than the original. I think that makes it more enjoyable to watch. The acting is excellent, with a lot of great performances by both Mar
}

\end{tcolorbox}
\end{minipage}
\caption{Comparison of IMDB review completions (constrained to 60 tokens) from $\pi_{\text{ref}}$ vs PITA.  PITA is able to change the sentiment positive, while $\pi_{\text{ref}}$ fails to do so.}  
\label{fig:imdb_example}
\end{figure}

\section{Additional Experimental Details}\label{app:additonal-experiments}

\subsection{Ablation Study on PITA training for GSM8K}

We benchmark \textbf{PITA} on the GSM8K grade-school arithmetic corpus~\cite{cobbe2021training}, which was also considered in the main paper. In contrast to the main paper, we evaluate PITA with only a single round of training i.e. ($k=1$ in Algorithm 1). We present our results in Table~\ref{tab:gsm8k_app_ablation}

\begin{table}[H]
    \centering
    \caption{Performance on the \textsc{GSM8K} \emph{test} split.}
    \label{tab:gsm8k_app_ablation}
    \begingroup
    \setlength{\tabcolsep}{6pt}
    \begin{tabular}{lccc}
        \toprule
        \textbf{Algorithm} & \textbf{pass@1} $\uparrow$ & \textbf{maj@8} $\uparrow$ & \textbf{KL}  \\
        \midrule
        $\pi_{\mathrm{ref}}$ & 69 & 85 & - \\
        Q\# & 78.4 & 88.1 & 2.65 \\
        Q\#-HF & 49.9 & 72.0 & 18.85 \\
        \midrule
        PITA (ours)  & 74.8 & 86.1 & 18.4 \\
        \bottomrule
    \end{tabular}
    \endgroup
\end{table}

We see a clear drop in performance when trained with limited data in both pass@1, and maj@8 metrics. Although the choice of two rounds ($k=2$ in Algorithm 1) was made to achieve a fairer comparison with Q\#, an additional round  of training allows for a greater diversity of preference data. The iterative training process is a key step in enabling PITA to better adapt to the preferences, leading to more robust results in alignment tasks.

\subsection{Reward Training for Q\#-HF}

Recall that Q\#-HF trains a reward model as an intermediate step for using the Q\# algorithm. We would like to highlight the sensitivity of algorithmic performance of reward training methods.
The results from Table~\ref{tab:gsm8k_app_ablation} involve a single round of training. That is, to ensure a fair comparison with other algorithms, the reward function for single-round Q\#-HF was trained using only 50\% of the data used in the main paper, so that the total data across reward estimation and policy training remained comparable. We observe a substantial drop in performance of the Q\# baseline which can we attributed to the poor reward separation between correct and incorrect generations as shown below.


\textbf{Reward-Model Calibration:} Figure~\ref{fig:reward_hist} contrasts reward scores assigned by the BT model to correct versus incorrect generations produced by the baseline. A larger separation indicates stronger alignment between the reward and task success.

\begin{figure}[H]
    \centering
    \includegraphics[width=\linewidth]{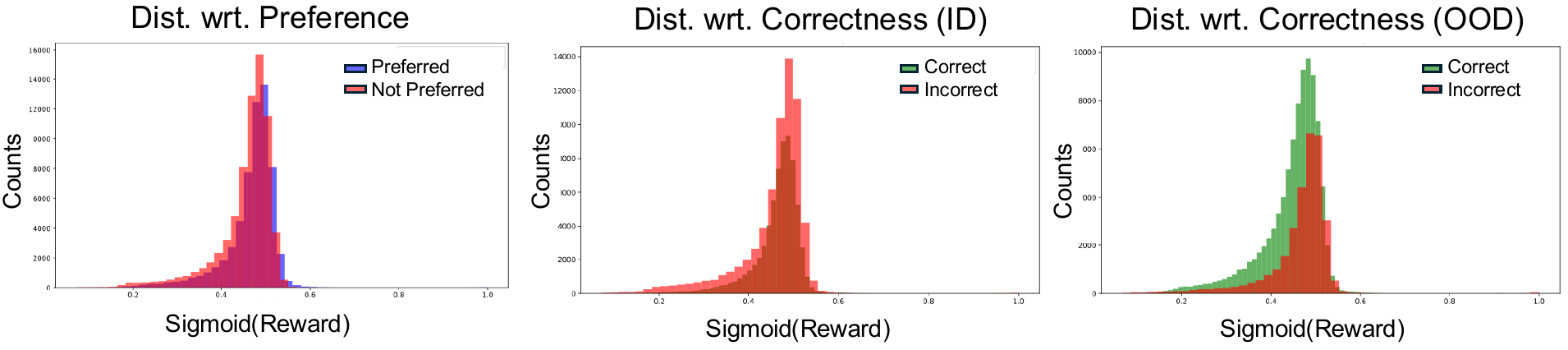}
    \caption{Distribution of learned reward scores w.r.t preferred/correct (green) and not preferred/incorrect (red) for ID(in-dist.)/OOD(out-of-dist.) data.  The reward function is difficult to learn.}
    \label{fig:reward_hist}
\end{figure}

\clearpage
\bibliographystyle{plainnat}
\bibliography{biblio}

\end{document}